%% file: main.tex
\documentclass[11pt]{article}
\pdfoutput=1

\usepackage{mathrsfs}
\usepackage{amsmath,amssymb}
\usepackage{bm}
\usepackage{natbib}
\usepackage[usenames]{color}
\usepackage{amsthm}

\usepackage{multirow} 
\usepackage{enumitem}

\usepackage[colorlinks,
linkcolor=red,
anchorcolor=blue,
citecolor=blue
]{hyperref}

\usepackage{setspace}
\usepackage[left=1in, right=1in, top=1in, bottom=1in]{geometry}

\usepackage{xcolor}

\title{\huge A Generalized Neural Tangent Kernel Analysis for Two-layer Neural Networks}

\author
{
    Zixiang Chen\thanks{Department of Computer Science, University of California, Los Angeles, CA 90095, USA; e-mail: {\tt chenzx19@cs.ucla.edu}} 
	~~~and~~~
	Yuan Cao\thanks{Department of Computer Science, University of California, Los Angeles, CA 90095, USA; e-mail: {\tt yuancao@cs.ucla.edu}} 
	~~~and~~~
	Quanquan Gu\thanks{Department of Computer Science, University of California, Los Angeles, CA 90095, USA; e-mail: {\tt qgu@cs.ucla.edu}}
	~~~and~~~
	Tong Zhang\thanks{Department of Computer Science and Mathematics, The Hong Kong University of Science and Technology, Hong Kong, China; e-mail: {\tt tongzhang@tongzhang-ml.org}}
}

\date{}

\usepackage{mylatexstyle}

\def\supp{\mathop{\text{supp}}}

\newcommand{\la}{\langle}
\newcommand{\ra}{\rangle}

\def\supp{\mathop{\text{supp}}}

\begin{document}

\maketitle

\input{Meanfield.tex}
\appendix

\input{appendix.tex}

\bibliography{deeplearningreference}
\bibliographystyle{ims}

\end{document}

%% file: Meanfield.tex
\begin{abstract}


A recent breakthrough in deep learning theory shows that the training of over-parameterized deep neural networks can be characterized by a kernel function called \textit{neural tangent kernel} (NTK). However, it is known that this type of results does not perfectly match the practice, as NTK-based analysis requires the network weights to stay very close to their initialization throughout training, and cannot handle regularizers or gradient noises. In this paper, we provide a generalized neural tangent kernel analysis and show that noisy gradient descent with weight decay can still exhibit a ``kernel-like'' behavior. This implies that the training loss converges linearly up to a certain accuracy. We also establish a novel generalization error bound for two-layer neural networks trained by noisy gradient descent with weight decay. 


\end{abstract}

\section{Introduction}
Deep learning has achieved tremendous practical success in a wide range of machine learning tasks \citep{krizhevsky2012imagenet,hinton2012deep,silver2016mastering}. However, due to the nonconvex and over-parameterized nature of modern neural networks, the success of deep learning cannot be fully explained by conventional optimization and machine learning theory.

A recent line of work studies the learning of over-parameterized neural networks in the so-called ``neural tangent kernel (NTK) regime'' \citep{jacot2018neural}. It has been shown that the training of over-parameterized deep neural networks can be characterized by the training dynamics of kernel regression with the \textit{neural tangent kernel} (NTK). Based on this, fast convergence rates can be proved for over-parameterized neural networks trained with randomly initialized (stochastic) gradient descent \citep{du2018gradient,allen2018convergence,du2018gradientdeep,zou2019gradient,zou2019improved}. Moreover, it has also been shown that target functions in the NTK-induced reproducing kernel Hilbert space (RKHS) can be learned by wide enough neural networks with good generalization error \citep{arora2019fine,arora2019exact,cao2019generalizationsgd}.


Despite having beautiful theoretical results, the NTK-based results are known to have their limitations, for not perfectly matching the empirical observations in many aspects. Specifically, NTK-based analysis requires that the network weights stay very close to their initialization throughout the training. Moreover, due to this requirement, NTK-based analysis cannot handle regularizers such as weight decay, or large additive noises in the noisy gradient descent.

Given the advantages and disadvantages of the existing NTK-based results, a natural question is:
\begin{center}
\emph{
Is it possible to establish the NTK-type results under more general settings?
}    
\end{center}
In this paper, we give an affirmative answer to this question by utilizing a mean-field analysis \citep{chizat2018global,mei2018mean,mei2019mean,wei2018regularization,fang2019over} to study neural tangent kernel. We show that with appropriate scaling, two-layer neural networks trained with noisy gradient descent and weight decay can still enjoy the nice theoretical guarantees. 

We summarize the contributions of our paper as follows:
\begin{itemize}[leftmargin = *]
\item 
Our analysis demonstrates that neural network training with noisy gradient and appropriate regularizers can still exhibit similar training dynamics as kernel methods, which is considered intractable in the neural tangent kernel literature, as the regularizer can easily push the network parameters far away from the initialization. Our analysis overcomes this technical barrier by relaxing the requirement on the closeness in the parameter space to the closeness in the distribution space. 
A direct consequence of our analysis is the linear convergence of noisy gradient descent up to certain accuracy for regularized neural network training. 
\footnote{Although we focus on the continuous-time limit of the noisy gradient descent algorithm, our result can be extended to the discrete-time setting by applying the approximation results in \citet{mei2018mean}}


\item We establish generalization bounds for the neural networks trained with noisy gradient descent with weight decay regularization. Our result shows that the infinitely wide neural networks trained by noisy gradient descent with weight decay can learn a class of functions that are defined based on a bounded $\chi^2$-divergence to initialization distribution. Different from standard NTK-type generalization bounds \citep{allen2018learning,arora2019fine,cao2019generalizationsgd}, our result can handle explicit regularization. Moreover, our proof  is based on an extension of the proof technique in \citet{meir2003generalization} 
from discrete distributions to continuous distributions, which may be of independent interest.

\end{itemize}

\noindent \textbf{Notation} We use lower case letters to denote scalars, and use lower and upper case bold face letters to denote vectors and matrices respectively. 
For a vector $\xb = (x_1,\dots,x_d)^\top \in \RR^d$, and any positive integer $p$, we denote the $\ell_p$ norm of $\xb$ as
$\|\xb\|_p=\big(\sum_{i=1}^d |x_i|^p\big)^{1/p}$. 
For a matrix $\Ab = (A_{ij})\in \RR^{m\times n}$, we denote by  $\| \Ab \|_2$ and $\|\Ab\|_F$ its spectral and Frobenius norms respectively.  We also define $\| \Ab \|_{\infty,\infty} = \max\{ |A_{ij}| : 1 \leq i \leq m, 1\leq j \leq n \}$. For a positive semi-definite matrix $\Ab$, we use $\lambda_{\min}(\Ab)$ to denote its smallest eigenvalue.


For a positive integer $n$, we denote $[n] = \{1,\ldots,n\}$.  We also use the following asymptotic notations. 
For two sequences $\{a_n\}$ and $\{b_n\}$, we write $a_n = O(b_n)$ if there exists an absolute constant $C$ such that $a_n\le C b_n$. We also introduce $\tilde O(\cdot)$ to  hide the logarithmic terms in the Big-O notations.

At last, for two distributions $p$ and $p'$, 
we define the Kullback–Leibler divergence (KL-divergence) and $\chi^2$-divergence between $p$ and $p'$ as follows:
\begin{align*}
    D_{\text{KL}}(p || p') = \int p(\zb)\log\frac{p(\zb)}{p'(\zb)}  d\zb,~~ D_{\chi^2}(p || p') = \int  \bigg(\frac{p(\zb)}{p'(\zb)} - 1\bigg)^2 p'(\zb) d\zb.
\end{align*}


\section{Related Work}
Our work is motivated by the recent study of neural network training in the ``neural tangent kernel regime''. In particular, \citet{jacot2018neural} first introduced the concept of neural tangent kernel by studying the training dynamics of neural networks with square loss. Based on neural tangent kernel, \citet{allen2018convergence,du2018gradientdeep,zou2019gradient} proved the global convergence of (stochastic) gradient descent under various settings. Such convergence is later studied by a line of work \citep{zou2019improved} with improved network width conditions in various settings. \citet{su2019learning,cao2019towards} studied the convergence along different eigendirections of the NTK. \citet{chizat2018note} extended the similar idea to a more general framework called ``lazy training''. \citet{liu2020toward} studied the optimization for over-parameterized systems of non-linear equations.  \citet{allen2018learning,arora2019fine,cao2019generalization,cao2019generalizationsgd} established generalization bounds for over-parameterized neural networks trained by (stochastic) gradient descent. \citet{li2019towards} studied noisy gradient descent with a certain learning rate schedule for a toy example.

Our analysis follows the mean-field framework adopted in the recent line of work \citep{bach2017breaking,chizat2018global,mei2018mean,mei2019mean,wei2018regularization,fang2019over,fang2019convexformulation}. \citet{bach2017breaking} studied the generalization performance of infinitely wide two-layer neural networks under the mean-field setting. \citet{chizat2018global} showed the convergence of gradient descent for training infinitely wide, two-layer networks under certain structural assumptions. \citet{mei2018mean} proved the global convergence of noisy stochastic gradient descent and established approximation bounds between finite and infinite neural networks. \citet{mei2019mean} further showed that this approximation error can be independent of the input dimension in certain cases, and proved that under certain scaling condition, the residual dynamics of noiseless gradient descent is close to the dynamics of NTK-based kernel regression within certain bounded time interval $[0,T]$. \citet{wei2018regularization} proved the convergence of a certain perturbed Wasserstein gradient flow, and established a generalization bound of the global minimizer of weakly regularized logistic loss. \citet{fang2019over,fang2019convexformulation} proposed a new concept called neural feature repopulation and extended the mean-field analysis.  

\section{Problem Setting and Preliminaries}

In this section we introduce the basic problem setting for training an infinitely wide two-layer neural network, and explain its connection to the training dynamics of finitely wide neural networks.

Inspired by the study in \citet{chizat2018note,mei2019mean}, we introduce a scaling factor $\alpha>0$ and study two-layer, infinitely wide neural networks of the form
\begin{align}\label{eq:def_infiniteNN}
    f(p,\xb) = \alpha\int_{\RR^{d+1}}uh(\btheta, \xb)p(\btheta,u)d\btheta du,
\end{align}
where $\xb\in \RR^d$ is the input, $\btheta\in \RR^d$ and $u\in \RR$ are the first and second layer parameters respectively, $p(\btheta,u)$ is their joint distribution, and $h(\btheta,\xb)$ is the activation function. 
It is easy to see that \eqref{eq:def_infiniteNN} is the infinite-width limit of the following neural network of finite width
\begin{align}\label{eq:def_finiteNN}
    f_{m}( \{ (\btheta_j,u_j) \}_{j=1}^m, \xb) = \frac{\alpha}{m} \sum_{j=1}^m u_j h(\btheta_j,\xb),
\end{align}
where $m$ is the number of hidden nodes, $\{ (\btheta_j,u_j) \}_{j=1}^m$ are i.i.d. samples drawn from $p(\btheta, u)$. Note that choosing $\alpha = \sqrt{m}$ in \eqref{eq:def_finiteNN} recovers the standard scaling in the neural tangent kernel regime \citep{du2018gradient}, and setting $\alpha = 1$ in \eqref{eq:def_infiniteNN} gives the standard setting for mean-field analysis \citep{mei2018mean,mei2019mean}.


We consider training the neural network with square loss and weight decay regularization. Let $S = \{ (\xb_1,y_1),\ldots, (\xb_n,y_n) \}$ be the training data set, and $\phi(y', y) = (y' - y)^2$ be the square loss function. We consider Gaussian initialization $p_0(\btheta,u) \propto \exp[-u^2/2 - \|\btheta\|_{2}^{2}/2]$. Then for finite-width neural network \eqref{eq:def_finiteNN}, we define the training objective function as
\begin{align}\label{eq:Qhat_def}
    \hat{Q}(\{ (\btheta_j,u_j) \}_{j=1}^m) = &\EE_{S}[\phi(f_{m}(\{ (\btheta_j,u_j) \}_{j=1}^m, \xb),y)] + \frac{\lambda}{m} \sum_{j=1}^{m} \bigg( \frac{u_j^{2}}{2} + \frac{\|\btheta_j\|_2^{2}}{2} \bigg),
\end{align}
where $\EE_{S}[\cdot]$ denotes the empirical average over the training sample $S$, and $\lambda>0$ is a regularization parameter. 
It is worth noting that when the network is wide enough, the neural network training is in the ``interpolation regime'', which gives zero training loss (first term in \eqref{eq:Qhat_def}). Therefore, even with a very large scaling parameter $\alpha$, weight decay (second term in \eqref{eq:Qhat_def}) is still effective. 




\begin{algorithm}[tb]
   \caption{Noisy Gradient Descent for Training Two-layer Networks}
   \label{alg:NGD}
\begin{algorithmic}
   \STATE {\bfseries Input:} Step size $\eta$, total number of iterations $T$
   \STATE Initialize $(\btheta_j,u_j) \sim p_0(\btheta,u)$, $j\in [m]$.
   \FOR{$t=0$ {\bfseries to} $T-1$}
   \STATE Draw Gaussian noises $\zeta_{u,j}\sim N(0, 2 \eta)$, $j\in[m]$
   \STATE $u_{t+1,j} = u_{t,j} - \eta \nabla_{u} \hat{Q}(\{ (\btheta_{t,j},u_{t,j}) \}_{j=1}^m) - \sqrt{\lambda}\zeta_{u,j}$
   \STATE Draw Gaussian noises $\bzeta_{\btheta,j}\sim N(0, 2 \eta \Ib_{d})$, $j\in[m]$
   \STATE $\btheta_{t+1,j} = \btheta_{t,j} - \eta \nabla_{\btheta} \hat{Q}(\{ (\btheta_{t,j},u_{t,j}) \}_{j=1}^m) - \sqrt{\lambda}\bzeta_{\btheta,j}$
   \ENDFOR
\end{algorithmic}
\end{algorithm}


In order to minimize the objective function $\hat{Q}(\{ (\btheta_j,u_j) \}_{j=1}^m) $ for the finite-width neural network~\eqref{eq:def_finiteNN}, we consider the noisy gradient descent algorithm, which is displayed in Algorithm~\ref{alg:NGD}.
It has been extensively studied \citep{mei2018mean,chizat2018global,mei2019mean,fang2019over} in the mean-field regime that, the continuous-time, infinite-width limit of Algorithm~\ref{alg:NGD} can be characterized by the following partial differential equation (PDE) of the distribution $p_t(\btheta,u)$\footnote{Throughout this paper, we define $\nabla$ and $\Delta$ without subscripts as the gradient/Laplacian operators with respect to the full parameter collection $(\btheta, u)$.}:
\begin{align}
\frac{dp_{t}(\btheta, u)}{dt} &= - \nabla_{u}[p_{t}(\btheta, u)g_{1}(t, \btheta, u)] - \nabla_{\btheta}\cdot[p_{t}(\btheta, u)g_{2}(t, \btheta, u)]  + \lambda \Delta[p_{t}(\btheta, u)], \label{eq:pde_p_t_1}
\end{align}
where
\begin{align*}
    &g_{1}(t, \btheta, u) =  -\alpha\EE_{S}[\nabla_{y'}\phi(f(p_{t},\xb),y)h(\btheta,\xb)] - \lambda u,\\
    &g_{2}(t, \btheta, u) =  -\alpha\EE_{S}[\nabla_{y'}\phi(f(p_{t},\xb),y)u\nabla_{\btheta}h(\btheta,\xb)] - \lambda \btheta.
\end{align*}
Below we give an informal proposition to describe the connection between Algorithm~\ref{alg:NGD} and the PDE \eqref{eq:pde_p_t_1}. One can refer to \citet{mei2018mean,chizat2018global,mei2019mean} for more details on such approximation results.

\begin{proposition}[informal]\label{prop:discreteapproximation}
Suppose that $h(\btheta,\xb)$ is sufficiently smooth, and PDE \eqref{eq:pde_p_t_1} has a unique solution $p_t$. Let $\{(\btheta_{t,j},u_{t,j})\}_{j=1}^m$, $t\geq 0$ be output by Algorithm~\ref{alg:NGD}. 
Then for any $t\geq 0$ and any $\xb$, it holds that
\begin{align*}
    \lim_{m\rightarrow \infty}\lim_{\eta \rightarrow 0} f_{m}( \{ (\btheta_{\lfloor t/\eta \rfloor, j},u_{\lfloor t/\eta \rfloor,j}) \}_{j=1}^m, \xb) = f(p_t,\xb).
\end{align*}
\end{proposition}
Based on Proposition~\ref{prop:discreteapproximation}, one can convert the original optimization dynamics in the parameter space to the distributional dynamics in the probability measure space. In the rest of our paper, we mainly focus on $p_t(\btheta,u)$ defined by the PDE \eqref{eq:pde_p_t_1}. It is worth noting that PDE \eqref{eq:pde_p_t_1} 
minimizes the following energy functional
\begin{align}\label{eq:definition_energyfunctional}
Q(p) = L(p) + \lambda  D_{\mathrm{KL}}(p||p_{0}),
\end{align}
where $L(p) = \EE_{S}[\phi\big((f(p,\xb), y\big)]$ is the empirical square loss, and $D_{\mathrm{KL}}(p||p_{0}) = \int p\log(p/p_{0})d\btheta du$ is the KL-divergence between $p$ and $p_0$ \citep{fang2019over}. 
The asymptotic convergence of
PDE \eqref{eq:pde_p_t_1} towards the global minimum of \eqref{eq:definition_energyfunctional} is recently established \citep{mei2018mean,chizat2018global,mei2019mean,fang2019over}. 

Recall that, compared with the standard mean-field analysis, we consider the setting with an additional scaling factor $\alpha$ in \eqref{eq:def_infiniteNN}. When $\alpha$ is large, we expect to build a connection to the recent results in the ``neural tangent kernel regime'' \citep{mei2019mean,chizat2018note}, where the neural network training is 
similar to kernel regression using the neural tangent kernel 
$K(\xb,\xb')$ defined as $K(\xb,\xb') = K_1(\xb,\xb') + K_2(\xb,\xb')$, where
\begin{align*}
    &K_1(\xb,\xb') = \int  u^2 \la \nabla_{\btheta} h(\btheta,\xb) , \nabla_{\btheta} h(\btheta,\xb') \ra p_0(\btheta,u) d\btheta du, \\
    &K_2(\xb,\xb') = \int h(\btheta,\xb)  h(\btheta,\xb')  p_0(\btheta,u) d\btheta du.
\end{align*}
Note that the neural tangent kernel function $K(\xb,\xb')$ is defined based on the initialization distribution $p_0$. This is because the specific network scaling in the neural tangent kernel regime forces the network parameters to stay close to initialization in the ``node-wise'' $\ell_2$ distance. In our analysis, we extend the definition of neural tangent kernel function to any distribution $p$, and define the corresponding Gram matrix $\Hb \in \RR^{n \times n}$ of the kernel function on the training sample $S$ as follows:
\begin{align}\label{eq:def_H}
\Hb(p) = \Hb_{1}(p) + \Hb_{2}(p), 
\end{align}
where $\Hb_{1}(p)_{i,j} = \EE_{p}[u^{2}\la \nabla_{\btheta}h(\btheta,\xb_{i}), \nabla_{\btheta}h(\btheta,\xb_{j} ) \ra]$ and $\Hb_{2}(p)_{i,j} = \EE_{p}[h(\btheta,\xb_{i})h(\btheta,\xb_{j})]$.
Note that our definition of the Gram matrix $\Hb$ is consistent with a similar definition in \citet{mei2019mean}.




\section{Main Results}\label{section:mainresults}
In this section we present our main results on the optimization and generalization of infinitely wide two-layer neural networks trained with noisy gradient descent in Algorithm \ref{alg:NGD}. 






We first introduce the following two assumptions, which are required in both optimization and generalization error analyses.

\begin{assumption}\label{assmption:xnorm}
The data inputs and responses are bounded: $\| \xb_i \|_2 \leq 1$, $|y_i| \leq 1$ for all $i\in [n]$.
\end{assumption}
Assumption~\ref{assmption:xnorm} is a natural and mild assumption. Note that this assumption is much milder than the commonly used assumption $\| \xb_i \|_2 = 1$ in the neural tangent kernel literature \citep{du2018gradient,allen2018convergence,zou2019gradient}. We would also like to remark that the bound $1$ is not essential, and can be replaced by any positive constant.

\begin{assumption}\label{assmption0}
The activation function has the form $h(\btheta, \xb) = \tilde{h}(\btheta^{\top}\xb)$, where $\tilde{h}(\cdot)$ is a three-times differentiable function that satisfies the following smoothness properties:
\begin{align*}
    |\tilde{h}(z)|\leq G_{1},~~~ |\tilde{h}'(z)|\leq G_{2},~~~
    |\tilde{h}''(z)|\leq G_{3}, ~~~|\big(z\tilde{h}'(z)\big)'|\leq G_{4},~~~|\tilde{h}'''(z)|\leq G_{5},
\end{align*}
where $G_{1},\ldots,G_{5}$ are absolute constants, and we set $G = \max\{G_{1},\ldots,G_{5}\}$ to simplify the bound.
\end{assumption}
$h(\btheta, \xb) = \tilde{h}(\btheta^{\top}\xb)$ is of the standard form in practical neural networks, and similar smoothness assumptions on $\tilde h(\cdot)$ are standard in the mean field literature \citep{mei2018mean, mei2019mean}. 
Assumption~\ref{assmption0} is satisfied by many smooth activation functions including the sigmoid and hyper-tangent functions. 


\subsection{Optimization Guarantees}
In order to characterize the optimization dynamics defined by PDE \eqref{eq:pde_p_t_1}, we need the following additional assumption.
\begin{assumption}\label{assumption:1}
The Gram matrix of the neural tangent kernel defined in \eqref{eq:def_H} is positive definite: $\lambda_{\min}(\Hb(p_0)) = \Lambda >0$. 
\end{assumption}
Assumption~\ref{assumption:1} is a rather weak assumption. In fact, \citet{jacot2018neural} has shown that if $\| \xb_i \|_2 = 1$ for all $i\in [n]$, Assumption~\ref{assumption:1} holds as long as each pair of training inputs $\xb_1,\ldots,\xb_n$ are not parallel. 

Now we are ready to present our main result on the training dynamics of infinitely wide neural networks.

\begin{theorem}\label{thm:main}

Let $\lambda_{0} = \sqrt{\Lambda/n}$ and suppose that PDE \eqref{eq:pde_p_t_1} has a unique solution $p_t$. Under Assumptions~\ref{assmption:xnorm}, \ref{assmption0} and \ref{assumption:1}, if 
\begin{align}
\alpha \geq 8\sqrt{A_{2}^2 + \lambda A_{1}^2} \cdot \lambda_{0}^{-2}R^{-1}, \label{condition1}
\end{align}
where $R = \min\Big\{ \sqrt{d + 1} ,  [\mathrm{poly}(G , \log(1/\lambda_{0}))]^{-1} \lambda_{0}^2 \Big\}$, then for all $t\in [0,+\infty)$, the following result hold:
\begin{align*}
&L(p_{t}) \leq 2\exp(-2\alpha^{2}\lambda_{0}^{2}t) + 2A_{1}^2\lambda^2\alpha^{-2}\lambda_{0}^{-4},\\
& D_{\mathrm{KL}}(p_{t}||p_{0}) \leq 4A_{2}^2\alpha^{-2}\lambda_{0}^{-4} + 4A_{1}^2\lambda\alpha^{-2}\lambda_{0}^{-4},
\end{align*}
where $A_{1} = 2G(d+1)+ 4G\sqrt{d + 1}$ and $A_{2} =16G\sqrt{d+1} + 4G$. 
\end{theorem}


Theorem~\ref{thm:main} shows that the loss of the neural network converges linearly up to $O(\lambda^2 \lambda_0^{-4} \alpha^{-2})$ accuracy, and the convergence rate depends on the smallest eigenvalue of the NTK Gram matrix. This matches the results for square loss in the neural tangent kernel regime \citep{du2018gradient}. However, we would like to emphasize that the algorithm we study here is noisy gradient descent, and the objective function involves a weight decay regularizer, both of which cannot be handled by the standard technical tools used in the NTK regime \citep{allen2018convergence,du2018gradient,du2018gradientdeep,zou2019improved}. Theorem~\ref{thm:main} also shows that the KL-divergence between $p_t$ and $p_0$ is bounded and decreases as $\alpha$ increases. This is analogous to the standard NTK results \cite{du2018gradient,allen2018convergence,zou2019gradient} where the Euclidean distance between the parameter returned by (stochastic) gradient descent and its initialization is bounded, and decreases as the network width increases.


The results in Theorem~\ref{thm:main} can also be compared with an earlier attempt by \citet{mei2019mean}, which uses mean-field analysis to explain NTK. 
While \citet{mei2019mean} only reproduces the NTK-type results without regularization, our result holds for a more general setting with weight decay and noisy gradient. Another work by \citet{tzen2020mean} uses mean-field analysis to study the lazy training of two-layer network. They consider a very small variance in parameter initialization, which is quite different from the practice of neural network training. In contrast, our work uses standard random initialization, and exactly follows the lazy training setting with scaling factor $\alpha$ \citep{chizat2018note}. Moreover, \citet{tzen2020mean} only 
characterize the properties of the optimal solution without finite-time convergence result, while we characterize the whole training process with a linear convergence rate.

\subsection{Generalization Bounds}

Next, we study the generalization performance of the neural network obtained by minimizing the energy functional $Q(p)$. For simplicity, we consider the binary classification problem, and use the $0$-$1$ loss $\ell^{\text{0-1}}(y',y) := \ind\{ y' y < 0 \}$ to quantify the errors of the network, where $ \ind\{ \cdot \}$ denotes the indicator function.

The following theorem presents the generalization bound for neural networks trained by Algorithm \ref{alg:NGD}. 

\begin{theorem}\label{thm:generalization_distribution_assumption}
Suppose that the training data $\{(\xb_i,y_i)\}_{i=1}^n$ are i.i.d. sampled from an unknown but fixed distribution $\cD$, and there exists a true distribution $p_{\text{true}}$ with $D_{\chi^2}(p_{\text{true}} || p_0) < \infty$, such that
\begin{align*}
    y = \int u h(\btheta,\xb) p_{\text{true}}(\btheta, u) d \btheta d u
\end{align*}
 for all $(\xb,y)\in \supp(\cD)$.
 Let $p^*$ be the minimizer of the energy functional \eqref{eq:definition_energyfunctional}. Under Assumptions~\ref{assmption:xnorm} and \ref{assmption0}, if $\alpha \geq \sqrt{n  \lambda } > 0$, then 
for any $\delta > 0$, with probability at least $1 - \delta$,
\begin{align*}
    &\EE_{\cD} [\ell^{\text{0-1}}(f(p^*,\xb),y)] \leq 
    (8 G + 1) \sqrt{\frac{ D_{\chi^2}(p_{\text{true}} || p_0) }{n}} + 6\sqrt{\frac{\log(2/\delta)}{2n}}.
\end{align*}
\end{theorem}

Theorem~\ref{thm:generalization_distribution_assumption} gives the generalization bound for the global minimizer of the energy functional $Q(p)$ obtained by noisy gradient descent with weight decay. We can see that it gives a standard $1/\sqrt{n}$ error rate as long as $p_{\text{true}}$ has a constant $\chi^2$-divergence to $p_0$. Moreover, the $\chi^2$-divergence $ D_{\chi^2}(p_{\text{true}} || p_0)$ also quantifies the difficulty for a target function defined by $p_{\text{true}}$ to be learnt. The larger  $ D_{\chi^2}(p_{\text{true}} || p_0)$ is, the more examples are needed to achieve the same target expected error.






Our generalization bound is different from existing NTK-based generalization results \citep{li2018learning,allen2018learning,arora2019fine,cao2019generalizationsgd}, which highly rely on the fact that the learned neural network weights are close to the initialization. Therefore, these generalization error bounds no longer hold with the presence of regularizer and are not applicable to our setting.  In addition, \citet{bach2017breaking} studied the generalization bounds for two-layer homogeneous networks and their connection to the NTK-induced RKHS.\footnote{Although it is not named ``neural tangent kernel'', the kernel function studied in \citep{bach2017breaking} is essentially NTK.} Our result based on the KL-divergence regularization is different from their setting and is not covered by their results.

\section{Proof Sketch of the Main Results}\label{section:proof_of_main_results}
In this section we present a proof sketch for Theorems~\ref{thm:main} and \ref{thm:generalization_distribution_assumption}.
\subsection{Proof Sketch of Theorem~\ref{thm:main}}
We first introduce the following definition of $2$-Wasserstein distance. For two distributions $p$ and $p'$ over $\RR^{d+1}$, we define 
\begin{align*}
    & \cW_2(p,p') =  \Bigg(\inf_{ \gamma \in \Gamma(p, p')} \int_{\RR^{d+1} \times \RR^{d+1}} \| \zb - \zb' \|_2^2 d\gamma(\zb,\zb') \Bigg)^{1/2},
\end{align*}
where $\Gamma(p, p')$ denotes the collection of all measures on $\RR^d \times \RR^d$ with marginals $p$ and $p'$ on the first and second factors respectively.


We also introduce the perturbation region $\cB(p_{0}, R) := \{p| \cW_{2}(p,p_{0})\leq R\} $ based on the Wasserstein distance to the initialization, where $R$ defined in Theorem~\ref{thm:main} gives the perturbation radius. We would like to highlight that compared with standard NTK-based analyses \cite{allen2018learning,arora2019fine,cao2019generalizationsgd} which are based on a perturbation region around initial weight parameter, our proof is based upon the $2$-Wasserstein neighborhood around $p_0$. Such an extension is essential to handle weight decay and gradient noises and is one of our key technical contributions.

The proof of Theorem~\ref{thm:generalization_distribution_assumption} can be divided into the following three steps.

\noindent\textbf{Step 1: Landscape properties when $p_t$ is close to $p_0$.}
We first consider the situation when the distribution $p_t$ is close to initial distribution $p_0$. 
\begin{lemma}\label{lemma:H}
Under Assumptions~\ref{assmption:xnorm}, \ref{assmption0} and \ref{assumption:1}, for any distribution $p$ with $\cW_2(p,p_0) \leq R$, 
we have $\lambda_{\min}(\Hb(p)) \geq \Lambda/2$, where $R$ is defined in Theorem \ref{thm:main}.
\end{lemma}
Lemma~\ref{lemma:H} shows that when $p_t$ is close to $p_0$ in $2$-Wasserstein distance, the Gram matrix at $p_t$ is strictly positive definite. This further implies nice landscape properties around $p_t$, which enables our analysis in the next step.

\noindent\textbf{Step 2: Loss and regularization bounds when $p_t$ is close to $p_0$.} 
With the results in Step 1, we  establish loss and regularization bounds when $p_t$ stays in $\cB(p_{0}, R)$ for some time period $[0,t^*]$, with
$
    t^{*} = \inf \{t\geq 0 : \mathcal{W}_{2}(p_{t}, p_{0}) > R\}
$.
We have $t^{*} = +\infty$ if the  $\{t\geq 0 : \mathcal{W}_{2}(p_{t}, p_{0}) > R\} = \emptyset$. The following lemma shows that the loss function decreases linearly in the time perioud $[0,t^*]$.
\begin{lemma}\label{lemma:Lconverge}
Under Assumptions~\ref{assmption:xnorm}, \ref{assmption0} and \ref{assumption:1}, for any $t \leq t^{*}$, it holds that
\begin{align*}
\sqrt{L(p_{t})} \leq \exp(-\alpha^2\lambda_{0}^{2}t)+ A_{1}\lambda\alpha^{-1}\lambda_{0}^{-2},
\end{align*}
where $A_1$ is defined in Theorem \ref{thm:main}.
\end{lemma}

Besides the bound on $L(p_t)$, we also have a bound on the KL-divergence between $p_t$ and $p_0$, as is given in the following lemma.
\begin{lemma}\label{lemma:W2}
Under Assumptions~\ref{assmption:xnorm}, \ref{assmption0} and \ref{assumption:1}, for any $t\leq t^{*}$, 
\begin{align*}
 D_{\mathrm{KL}}(p_{t}||p_{0}) \leq 4A_{2}^2\alpha^{-2}\lambda_{0}^{-4} + 4A_{1}^2\lambda\alpha^{-2}\lambda_{0}^{-4},
\end{align*}
where $A_1$ and $A_2$ are defined in Theorem \ref{thm:main}. 
\end{lemma}
Here we would like to remark that the bound in Lemma~\ref{lemma:W2} does not increase with time $t$, which is an important feature of our result. This is achieved by jointly considering two types of bounds on the KL-divergence between $p_t$ and $p_0$: the first type of bound is on the time derivative of $D_{\text{KL}}(p_t,p_0)$ based on the training dynamics described by \eqref{eq:pde_p_t_1}, and the second type of bound is a direct KL-divergence bound based on the monotonicity of the energy functional $Q(p_t)$. The detailed proof of this lemma is deferred to the appendix.

\noindent\textbf{Step 3: Large scaling factor $\alpha$ ensures distribution closeness throughout training.} 
When $\alpha$ is sufficiently large, $p_t$ will not escape from the perturbation region. To show this, we utilize the following Talagrand inequality (see Corollary~2.1 in \citet{otto2000generalization} and Theorem~9.1.6 in \citet{bakry2013analysis}), which is based on the fact that in our setting $p_0$ is a Gaussian distribution.
\begin{lemma}[\citet{otto2000generalization}]\label{lemma:W2toKL}
The probability measure $p_{0}(\btheta , u) \propto \exp[-u^2/2 -  \|\btheta\|_{2}^{2}/2]$ satisfies following Talagrand inequality 
\begin{align*}
\mathcal{W}_{2}(p,p_{0})\leq 2 \sqrt{D_{\mathrm{KL}}(p||p_{0}) }.
\end{align*}
\end{lemma}
The main purpose of Lemma~\ref{lemma:W2toKL} is to build a connection between the $2$-Wasserstein ball around $p_0$ and the KL-divergence ball. 

We are now ready to finalize the proof. Note that given our results in Step 2, it suffices to show that $t^*=+\infty$, which is proved based on a reduction to absurdity.
\begin{proof}[Proof of Theorem~\ref{thm:main}]
By the definition of $t^*$, for any $t \leq t^{*}$, we have 
\begin{align*}
\mathcal{W}_{2}(p_{t}, p_{0}) &\leq  2 D_{\mathrm{KL}}(p_{t}||p_{0})^{1/2} \leq 2\big(4A_{2}^2\alpha^{-2}\lambda_{0}^{-4} + 4A_{1}^2\lambda\alpha^{-2}\lambda_{0}^{-4}\big)^{1/2} \leq R/2,
\end{align*}
where the first inequality is by Lemma \ref{lemma:W2toKL}, the second inequality is by Lemma \ref{lemma:W2} ,and the third inequality is due to the choice of $\alpha$ in \eqref{condition1}.

This deduces that the set $ \{t \geq 0 : \mathcal{W}_{2}(p_{t}, p_{0}) > R\}$ is empty and $t^{*} = \infty$, because otherwise $\mathcal{W}_{2}(p_{t^*}, p_{0}) = R$ by the continuity of $2$-Wasserstein distance. Therefore the results of Lemmas \ref{lemma:Lconverge} and \ref{lemma:W2} hold for all $t\in [0,+\infty)$.
Squaring both sides of the inequality in Lemma~\ref{lemma:Lconverge} and applying Jensen's inequality gives
\begin{align*}
L(p_{t})\leq 2\exp(-2\alpha^2\lambda_{0}^{2}t)+ 2A_{1}^2\lambda^2\alpha^{-2}\lambda_{0}^{-4}.
\end{align*}
This completes the proof.
\end{proof}

\subsection{Proof Sketch of Theorem~\ref{thm:generalization_distribution_assumption}}
For any $M > 0$, we consider the following class of infinitely wide neural network functions characterized by the KL-divergence to initialization
\begin{align}\label{eq:def_KLfunctionclass}
    \cF_{\text{KL}}(M) = \big\{ f(p,\xb) : D_{\mathrm{KL}}(p \| p_0) \leq M \big\}.
\end{align}
Our proof consists of the following two steps.

\noindent\textbf{Step 1: A KL-divergence based Rademacher complexity bound.} Motivated by the KL-divergence regularization in the energy functional $Q(p)$, we first derive a Rademacher complexity bound for the function class $\cF_{\text{KL}}(M)$, which is given as the following lemma.
\begin{lemma}\label{lemma:rademacherbound_KL}
Suppose that $| h(\btheta,\xb) | \leq G$ for all $\btheta$ and $\xb$, and $M \leq 1/2$. Then
\begin{align*}
    \mathfrak{R}_n(\cF_{\text{KL}}(M)) \leq 2G\alpha  \sqrt{\frac{M}{n}}.
\end{align*}
\end{lemma}
Lemma~\ref{lemma:rademacherbound_KL} is another key technical contribution of our paper. Different from previous NTK-based generalization error analysis that utilizes the approximation of neural network functions with linear models \citep{lee2019wide,cao2019generalizationsgd}, we use mean-field analysis to directly bound the Rademacher complexity of neural network function class. At the core of the proof for Lemma~\ref{lemma:rademacherbound_KL} is an extension of \citet{meir2003generalization} for discrete distributions to continuous distributions, which is of independent interest.


\noindent\textbf{Step 2: Expected $0$-$1$ loss bound over $\cF_{\text{KL}}(M) $.}
We bound the expected $0$-$1$ loss by the square root of the empirical square loss function and the Rademacher complexity. 
\begin{lemma}\label{lemma:expected01bound}
For any $\delta >0$, with probability at least $1 - \delta$, the following bound holds uniformly over all $f\in \cF_{\text{KL}}(M)$:
\begin{align*}
    \EE_{\cD}[\ell^{0\text{-}1}( f(\xb) , y )]
    \leq \sqrt{\EE_{S}[| f(\xb) - y |^2]} + 4 \mathfrak{R}_n(\cF_{\text{KL}}(M)) + 6\sqrt{\frac{\log(2/\delta)}{2n}}.
\end{align*}
\end{lemma}
The bound in Lemma~\ref{lemma:expected01bound} utilizes the property of square loss instead of margin-based arguments \citep{bartlett2017spectrally}. In this way, we are able to obtain a tighter bound, as our setting uses square loss as the objective.


We now finalize the proof by deriving the loss and regularization bounds at $p^*$ and plugging them into Lemma~\ref{lemma:expected01bound}. 
\begin{proof}[Proof of Theorem~\ref{thm:generalization_distribution_assumption}]
Let $\overline{D} = D_{\chi^2}(p_{\text{true}} || p_0) < \infty$. 
Define 
$$
\hat p = \frac{\alpha - 1}{ \alpha } \cdot  p_0 + \frac{1}{ \alpha } \cdot p_{\text{true}}. 
$$
Then  we have $\int \hat p(\theta, u) du d\theta = 1$, $\hat p(\theta, u) \geq 0$, meaning that $\hat p$ is a well-defined density function. The training loss of $\hat p$ can be calculated as follows:
\begin{align}
    L(\hat p)&= \EE_S \bigg[ \alpha \int u h(\theta, \xb) \hat p(u,\theta) d u d\theta - y \bigg]^2  = \EE_S\bigg( 0 + \alpha\cdot \frac{ 1}{ \alpha } y - y \bigg)^2 = 0. \label{eq:generalization_distribution_assumption_proof_eq1}
\end{align}
Moreover,  by the fact that KL-divergence is upper bounded by the $\chi^2$-divergence, we have
\begin{align}
    D_{\text{KL}}(\hat p || p_0) \leq D_{\chi^2}(\hat p || p_0) 
    &= \int \bigg[ \frac{\alpha - 1}{\alpha } +  \frac{ p_{\text{true}}(\btheta,u)}{ \alpha p_0(\btheta,u)} - 1\bigg]^2 p_0 (\btheta,u) d\btheta du = \alpha^{-2}  \overline{D},\label{eq:generalization_distribution_assumption_proof_eq2}
\end{align}
where the last equation is by the definition of $\chi^2$-divergence.
Now we have
\begin{align*}
    Q(p^*) &\leq Q(\hat p) = L(\hat p) + \lambda D_{\text{KL}}(\hat p || p_0) \leq \alpha^{-2} \lambda  \overline{D},
\end{align*}
where the first inequality follows by the optimality of $p^*$, and we plug \eqref{eq:generalization_distribution_assumption_proof_eq1}, \eqref{eq:generalization_distribution_assumption_proof_eq2} into the definition of the energy function $Q(p)$ in \eqref{eq:definition_energyfunctional} to obtain the second inequality. Applying the definition of $Q(p)$ again gives the following two bounds:
\begin{align}
    & L(p^*) = \EE_{S}[| f(p^*,\xb) - y |^2] \leq \alpha^{-2} \lambda \overline{D}, \label{eq:generalization_distribution_assumption_proof_eq3}\\
    & D_{\text{KL}}( p^* || p_0) \leq \alpha^{-2} \overline{D}.\label{eq:generalization_distribution_assumption_proof_eq4}
\end{align}
By \eqref{eq:generalization_distribution_assumption_proof_eq4}, we have $f(p^*,\xb) \in \cF_{\text{KL}}( \alpha^{-2} \hat{D} )$. Therefore, applying Lemma~\ref{lemma:expected01bound} with $M =  \alpha^{-2} \hat{D}$ gives
\begin{align*}
    \EE_{\cD}[\ell^{0\text{-}1}( f(p^*,\xb) , y )]
    &\leq \sqrt{\EE_{S}[| f(p^*,\xb) - y |^2]} + 4 \mathfrak{R}_n(\cF_{\text{KL}}( \alpha^{-2} \hat{D})) + 6\sqrt{\frac{\log(2/\delta)}{2n}} \\
    &\leq \sqrt{\alpha^{-2} \lambda \overline{D}} + 8G\alpha  \sqrt{\frac{\alpha^{-2} \hat{D}}{n}} + 6\sqrt{\frac{\log(2/\delta)}{2n}}\\
    &\leq (8G + 1)\sqrt{\frac{ \hat{D}}{n}} + 6\sqrt{\frac{\log(2/\delta)}{2n}},
\end{align*}
where the second inequality follows from \eqref{eq:generalization_distribution_assumption_proof_eq3} and Lemma~\ref{lemma:rademacherbound_KL}, and the third inequality follows from the assumption that $\alpha \geq \sqrt{n  \lambda }$. This finishes the proof.
\end{proof}

\section{Conclusion}
In this paper, we demonstrate that the neural tangent kernel based regression can characterize neural network training dynamics in a general setting where weight decay and gradient noises are implemented. This leads to the linear convergence of noisy gradient descent up to certain accuracy. Compared with existing analysis in the neural tangent kernel regime, our work points out an important observation that as long as the distribution of parameters stays close to the initialization, it does not matter whether the parameters themselves are close to their initial values. We also establish a novel generalization bound for the neural network trained by noisy gradient descent with weight decay regularization.

%% file: appendix.tex
\section{Proof of Lemmas in Section~\ref{section:proof_of_main_results}}
\label{appendix:proof_part1}
In this section we provide the proofs of lemmas we use in Section~\ref{section:proof_of_main_results} for the proof of our main results. We first introduce the following notations. We denote $\fb(t) = (f(p_{t},\xb_1), \ldots, f(p_{t},\xb_n) )^\top$. Moreover, we define
\begin{align}
    &\hat{g}_{1}(t, \btheta, u) = -\alpha\EE_{S}[\nabla_{f}\phi(f(p_{t},\xb),y)h(\btheta,\xb)], \label{eq:def_ghat1} \\
    &\hat{g}_{2}(t, \btheta, u) = -\alpha\EE_{S}[\nabla_{f}\phi(f(p_{t},\xb),y)u\nabla_{\btheta}h(\btheta,\xb)]. \label{eq:def_ghat2}
\end{align}

\subsection{Proof of Lemma~\ref{lemma:H}}
Here we give the proof of Lemma~\ref{lemma:H}. The following lemma  summarizes some basic properties of the activation function $h(\btheta,u)$.

\begin{lemma}\label{lemma:new2oldassumption}
Under Assumptions \ref{assmption:xnorm} and \ref{assmption0}, for all $\xb$ and $\btheta$, it holds that $|h(\btheta, \xb)| \leq G$, $\|\nabla_{\btheta}h(\btheta, \xb)\|_{2} \leq G$, $|\Delta h(\btheta, \xb)| \leq G$, $\|\nabla_{\btheta}h(\btheta_{1}, x) - \nabla_{\btheta}h(\btheta_{2}, x)\|_{2} \leq G\|\btheta_{1} - \btheta_{2}\|_{2}$, $\|\nabla_{\btheta}\big(\nabla_{\btheta}h(\theta, \xb)\cdot \btheta\big)\|_{2}\leq G$, $\|\nabla_{\btheta} \Delta_{\btheta}h(\btheta, \xb)\|_{2}\leq G$.
\end{lemma}

We also give the following two lemmas to characterize the difference between the Gram matrices defined with $p_0$ and some other distribution $p$ that is close to $p_0$ in $2$-Wasserstein distance.
\begin{lemma}\label{lemma:H1}
Under Assumptions \ref{assmption:xnorm} and \ref{assmption0}, for any distribution $p$ with $\cW_2(p, p_{0}) \leq \sqrt{d + 1}$ and any $r > 0$,
\begin{align*}
\|\Hb_{1}(p) - \Hb_{1}(p_{0})\|_{\infty, \infty} &\leq G^2 \Big[\sqrt{8d + 10} + 2r^2G^{2}\Big]\cW_2(p, p_{0}) + 2G^2 \EE_{p_{0}}[u_{0}^2\ind(|u_{0}\geq r|)].
\end{align*}
\end{lemma}

\begin{lemma}\label{lemma:H2}
Under Assumptions \ref{assmption:xnorm} and \ref{assmption0}, for any distribution $p$ with $\cW_2(p, p_{0}) \leq \sqrt{d+1}$,
\begin{align*}
\|\Hb_{2}(p) -  \Hb_{2}(p_{0})\|_{\infty, \infty} \leq 2G^{2}\cW_2(p, p_{0}).
\end{align*}
\end{lemma}

The following lemma gives a tail bound with respect to our initialization distribution $p_0$, which we frequently utilize for truncation arguments.

\begin{lemma}\label{lemma: tail bound}
The initialization distribution $p_0$ satisfies the following tail bound:
\begin{align*}
\EE_{p_{0}}[u_{0}^2\ind(|u_{0}| \geq r)] \leq \frac{\exp(-r^2/4)}{2}.
\end{align*}
\end{lemma}

We are now ready to provide the proof of Lemma~\ref{lemma:H}. 

\begin{proof}[Proof of Lemma~\ref{lemma:H}]
Here we first give the definition of $R$ in Theorem~\ref{thm:main} with specific polynomial dependencies.
\begin{align*}
R &= \min\Big\{ \sqrt{d + 1} , [\mathrm{poly}(G, \log(n/\Lambda)) n/\Lambda]^{-1} \Big\}\\
&\leq \min\Big\{ \sqrt{d + 1} , \Big( 8G^2 \sqrt{8d + 10} + 64G^{2}\log(8\Lambda^{-1}nG^2) \Big)^{-1} n^{-1}\Lambda \Big\}.
\end{align*}

Note that the definition of $R$, the results for Lemmas~\ref{lemma:H1} and \ref{lemma:H2} hold for all $p$ with $\cW_2(p,p_0)\leq R$. Now by Lemma \ref{lemma:H1}, for any $p$ with $\cW_2(p, p_{0}) \leq R$ and any $r > 0$,
\begin{align}\label{eq:lemmaHproof_eq1}
 \|\Hb_{1}(p) - \Hb_{1}(p_{0})\|_{\infty, \infty} &\leq G^2 R \sqrt{8d + 10} + 2r^2G^{2}R + 2G^2 \EE_{p_{0}}[u_{0}^2\ind(|u_{0}\geq r|)].    
\end{align}

Choose $r = 2\sqrt{\log(8\Lambda^{-1}nG^2)}$, then by Lemma \ref{lemma: tail bound} we have
\begin{align}\label{eq:lemmaHproof_eq2}
\EE_{p_0}[u_{0}^2\ind(|u_{0}\geq r|)] \leq \frac{\Lambda}{16nG^2}.    
\end{align}

Moreover, by the definition of $R$, we have
\begin{align}\label{eq:lemmaHproof_eq3}
R \leq  \Big( 8G^2 \sqrt{8d + 10} + 16G^{2}r^2  \Big)^{-1} n^{-1}\Lambda.
\end{align}
Plugging the bounds on $\EE_{p_0}[u_{0}^2\ind(|u_{0}\geq r|)]$ and $R$ given by \eqref{eq:lemmaHproof_eq2} and \eqref{eq:lemmaHproof_eq3} into \eqref{eq:lemmaHproof_eq1} gives 
\begin{align}\label{eq:lemmaHproof_eq4}
 \|\Hb_{1}(p) - \Hb_{1}(p_{0})\|_{\infty, \infty} &\leq G^2 R \sqrt{8d + 10} + 2r^2G^{2}R + G^2 \EE_{p_{0}}[u_{0}^2\ind(|u_{0}\geq r|)]\\
 &\leq \frac{\Lambda}{8n} +  \frac{\Lambda}{8n}\\
 &= \frac{\Lambda}{4n}.
\end{align}

By Lemma \ref{lemma:H2}, for any distribution $p$ with $\cW_2(p, p_{0}) \leq R$,
\begin{align}\label{eq:lemmaHproof_eq5}
\|\Hb_{2}(p) -  \Hb_{2}(p_{0})\|_{\infty, \infty} \leq 2G^{2}R.
\end{align}

The definition of $R$ also leads to the following bound: 
\begin{align}\label{eq:lemmaHproof_eq6}
R \leq (8G^{2})^{-1} n^{-1}\Lambda .
\end{align}
Therefore we can plug the bound \eqref{eq:lemmaHproof_eq6} into \eqref{eq:lemmaHproof_eq5}, which gives 
\begin{align}\label{eq:lemmaHproof_eq7}
\|\Hb_{2}(p) -  \Hb_{2}(p_{0})\|_{\infty, \infty} \leq \frac{\Lambda}{4n}.    
\end{align}
Combining \eqref{eq:lemmaHproof_eq4} and \eqref{eq:lemmaHproof_eq7} further gives
\begin{align*}
\|\Hb(p) -  \Hb(p_{0})\|_{\infty, \infty} \leq \|\Hb_{1}(p) -  \Hb_{1}(p_{0})\|_{\infty, \infty} + \|\Hb_{2}(p) -  \Hb_{2}(p_{0})\|_{\infty, \infty} \leq \frac{\Lambda}{2n}.
\end{align*}
Then by standard matrix perturbation bounds, we have $\lambda_{\min}(\Hb(p)) \geq \lambda_{\min}(\Hb(p_{0})) - \|\Hb(p) -  \Hb(p_{0})\|_{2} \geq \lambda_{\min}(\Hb(p_{0})) - n\|\Hb(p) -  \Hb(p_{0})\|_{\infty, \infty} \geq \Lambda/2$, which finishes the proof.
\end{proof}

\subsection{Proof of Lemma~\ref{lemma:Lconverge}}

Here we give the proof of Lemma~\ref{lemma:Lconverge}. The following lemma summarizes some basic calculation on the training dynamics. Here we remind the readers that the definitions of $\hat{g}_{1}(t, \btheta, u)$ and $\hat{g}_{2}(t, \btheta, u)$ are given in \eqref{eq:def_ghat1} and \eqref{eq:def_ghat2} respectively. 

\begin{lemma}\label{lemma:Ldynamic}
Let $p_t$ be the solution of PDE \eqref{eq:pde_p_t_1}. Then the following identity holds.
\begin{align}
\frac{\partial L(p_{t})}{\partial t}  &= -\int_{\RR^{d+1}}p_{t}(\btheta, u) \|\hat{g}_{1}(t, \btheta, u)\|_2^2 d\btheta du - \int_{\RR^{d+1}}p_{t}(\btheta, u) |\hat{g}_{2}(t, \btheta, u)|^2 d\btheta du  \notag\\
&\qquad + \lambda\int_{\RR^{d+1}}p_{t}(\btheta, u)[\hat{g}_{1}\cdot u + \hat{g}_{2}\cdot \btheta - \nabla_{u} \cdot \hat{g}_{1} - \nabla_{\btheta} \cdot \hat{g}_{2} ] d\btheta du. \label{eq:Ldynamic}
\end{align}
\end{lemma}

Lemma~\ref{lemma:Ldynamic} decomposes the time derivative of $L(p_t)$ into several terms. The following two lemmas further provides bounds on these terms. Note that by the definition in \eqref{eq:def_ghat1} and \eqref{eq:def_ghat2}, Lemma~\ref{lemma: Residual Dynamic} below essentially serves as a bound on the first two terms on the right-hand side of \eqref{eq:Ldynamic}.

\begin{lemma}\label{lemma: Residual Dynamic}
Under Assumptions \ref{assmption:xnorm}, \ref{assmption0} and \ref{assumption:1}, let  $\lambda_{0}$ be defined in Theorem \ref{thm:main}. Then for $t\leq t^*$, it holds that
\begin{align*}
\int_{\RR^{d+1}}p_{t}(\btheta,u)\big[|\EE_{S}[(f(p_{t},\xb)- y)h(\btheta,\xb)]|^{2} + \|\EE_{S}[(f(p_{t},\xb)- y)u\nabla_{\btheta}h(\btheta,\xb)]\|_2^{2}\big]d\btheta du \geq \frac{\lambda_{0}^2}{2}L(p_{t}).
\end{align*}
\end{lemma}

\begin{lemma}\label{lemma:diffusion term1}
Under Assumptions \ref{assmption:xnorm} and \ref{assmption0}, let $A_{1}$ be defined in Theorem \ref{thm:main}. Then for $t \leq t^{*}$,it holds that
\begin{align*}
\int_{\RR^{d+1}}p_{t}(\btheta, u)[\hat{g}_{1}\cdot u + \hat{g}_{2}\cdot \btheta - \nabla_{u} \cdot \hat{g}_{1} - \nabla_{\btheta} \cdot \hat{g}_{2} ] d\btheta du \leq 2\alpha A_{1}\sqrt{L(p_{t})}.
\end{align*}
\end{lemma}

We now present the proof of Lemma~\ref{lemma:Lconverge}, which is based on the calculations in Lemmas~\ref{lemma:Ldynamic}, \ref{lemma: Residual Dynamic} and \ref{lemma:diffusion term1} as well as the application of Gronwall's inequality.

\begin{proof}[Proof of Lemma~\ref{lemma:Lconverge}]

By Lemma \ref{lemma:Ldynamic}, we have
\begin{align}
\frac{\partial L(p_{t})}{\partial t}  &= - \underbrace{\bigg[ \int_{\RR^{d+1}}p_{t}(\btheta, u) \|\hat{g}_{1}(t, \btheta, u)\|_2^2 d\btheta du + \int_{\RR^{d+1}}p_{t}(\btheta, u) |\hat{g}_{2}(t, \btheta, u)|^2 d\btheta d\btheta du \bigg] }_{I_1}   \notag\\
&\qquad + \underbrace{ \lambda\int_{\RR^{d+1}}p_{t}(\btheta, u)[\hat{g}_{1}\cdot u + \hat{g}_{2}\cdot \btheta - \nabla \cdot \hat{g}_{1} - \nabla \cdot \hat{g}_{2} ] d\btheta du}_{I_2}.\label{eq:RD2},
\end{align}
For $I_1$, we have
\begin{align}
I_1 &= 4\alpha^2\int_{\RR^{d+1}}p_{t}(\btheta,u)\big[|\EE_{S}[(f(p_{t},\xb)- y)h(\btheta,\xb)]|^{2}\\
&\qquad+ \|\EE_{S}[(f(p_{t},\xb)- y)u\nabla_{\btheta}h(\btheta,\xb)]\|_2^{2}\big]d\btheta du \\
&\geq 2\alpha^{2}\lambda_{0}^{2}L(p_{t}),\label{eq:RD1}
\end{align}
where the equation follows by the definitions of $\hat{g}_{1}(t, \btheta, u)$, $\hat{g}_{2}(t, \btheta, u)$ in \eqref{eq:def_ghat1},  \eqref{eq:def_ghat1}, and the inequality follows by Lemma \ref{lemma: Residual Dynamic}. 
For $I_2$, we directly apply Lemma~\ref{lemma:diffusion term1} and obtain
\begin{align}\label{eq:RD3}
    I_2 \leq 2A_{1}\alpha\lambda\sqrt{L(p_{t})}.
\end{align}
Plugging the bounds \eqref{eq:RD1} and \eqref{eq:RD3} into \eqref{eq:RD2} yields
\begin{align}\label{eq:RD4}
    \frac{\partial L(p_{t})}{\partial t} \leq -2\alpha^2\lambda_{0}^{2}L(p_{t}) + 2A_{1}\alpha\lambda\sqrt{L(p_{t})}.
\end{align}
Now denote $V(t) = \sqrt{L(p_{t})} - A_{1}\lambda\alpha^{-1}\lambda_{0}^{-2}$. Then \eqref{eq:RD4} implies that\footnote{ The derivation we present here works as long as  $ L(p_t) \neq 0 $. A more thorough but complicated analysis can deal with the case when $ L(p_t) = 0$ for some $t$. However for simplicity we omit the more complicated proof, since loss equaling to zero is a trivial case for a learning problem.}
\begin{align*}
\frac{\partial V(t)}{\partial t} \leq -\alpha^2\lambda_{0}^{2}V(t).
\end{align*}
By Gronwall's inequality we further get
\begin{align*}
V(t) \leq \exp(-\alpha^2\lambda_{0}^{2}t)V(0).
\end{align*}
By $V(0) = \sqrt{L(p_{0})} - A_{1}\lambda\alpha^{-1}\lambda_{0}^{-2} \leq \sqrt{L(p_{0})}\leq 1$, we have
\begin{align}
\sqrt{L(p_{t})} \leq  \exp(-\alpha^2\lambda_{0}^{2}t)+ A_{1}\lambda\alpha^{-1}\lambda_{0}^{-2}. \label{eq:RD5}
\end{align}
This completes the proof.
\end{proof}

\subsection{Proof of Lemma~\ref{lemma:W2}}
In this subsection we present the proof of Lemma~\ref{lemma:W2}.

\begin{lemma}\label{lemma:Dcontroll}
Under Assumptions \ref{assmption:xnorm}, \ref{assmption0} and \ref{assumption:1}, let $\lambda_{0}$ be defined in Theorem \ref{thm:main}. Then for $t\leq t^*$ the following inequality holds
\begin{align*}
D_{\mathrm{KL}}(p_{t}||p_{0}) &\leq 2A_{2}^2\alpha^{-2}\lambda_{0}^{-4} + 2A_{2}^2 A_{1}^2\lambda^2\lambda_{0}^{-4}t^2.
\end{align*}
\end{lemma}

If $\lambda \not= 0$, the KL distance bound given by Lemma \ref{lemma:Dcontroll}  depends on $t$, we can give a tighter bound by the monotonically deceasing property of $Q(p_{t})$ given by the following lemma, which states that the energy functional is monotonically decreasing during training. Note that this is not a new result, as it is to some extent an standard result, and has been discussed in \citet{mei2018mean,mei2019mean,fang2019over}.
\begin{lemma}\label{lemma:Qdynamic}
Let $p_t$ be the solution of PDE \eqref{eq:pde_p_t_1}. Then $Q(p_{t})$ is monotonically deceasing, i.e.,
\begin{align}
\frac{\partial Q(p_{t})}{\partial t} \leq 0.
\end{align}
\end{lemma}

\begin{proof}[Proof of Lemma~\ref{lemma:W2}]
Notice that for $\lambda = 0$, Lemma \ref{lemma:Dcontroll} directly implies the conclusion. So in the rest of the proof we consider the situation where  $\lambda > 0$. Denote $t_{0} = A_{1}^{-1}\alpha^{-1}\lambda^{-1}$, we consider two cases $t_{0} \geq t_{*}$ and $t_{0} < t_{*}$ respectively.

If $t_{0} \geq t_{*}$, then for $t \leq t^{*}$ we have $t \leq t_{0}$
\begin{align*}
D_{\mathrm{KL}}(p_{t}||p_{0}) &\leq 2A_{2}^2\alpha^{-2}\lambda_{0}^{-4} + 2A_{2}^2 A_{1}^2\lambda^2\lambda_{0}^{-4}t^2\\
&\leq 2A_{2}^2\alpha^{-2}\lambda_{0}^{-4} + 2A_{2}^2 A_{1}^2\lambda^2\lambda_{0}^{-4}t_{0}^2\\
&= 4A_{2}^2\alpha^{-2}\lambda_{0}^{-4}\\
&\leq 4A_{2}^2\alpha^{-2}\lambda_{0}^{-4} + 4A_{1}^2\lambda\alpha^{-2}\lambda_{0}^{-4},
\end{align*}
where the first inequality is by Lemma \ref{lemma:Dcontroll} and the second inequality is by  $t \leq t_{0}$.

If $t_{0} < t_{*}$, then for $t \leq t_{0}$, we also have 
\begin{align*}
D_{\mathrm{KL}}(p_{t}||p_{0}) \leq 4A_{2}^2\alpha^{-2}\lambda_{0}^{-4} \leq 4A_{2}^2\alpha^{-2}\lambda_{0}^{-4} + 4A_{1}^2\lambda\alpha^{-2}\lambda_{0}^{-4}.
\end{align*}

For $t_{0}<t\leq t_{*}$, consider $Q(p_{t}) = L(p_{t}) + \lambda  D_{\mathrm{KL}}(p_{t}||p_{0})$. The monotonically deceasing property of $Q(p_{t})$ in Lemma \ref{lemma:Qdynamic} implies that,
\begin{align}
D_{\mathrm{KL}}(p_{t}||p_{0}) &\leq \lambda^{-1}Q(p_{t}) \leq \lambda^{-1}Q(p_{t_{0}}). \label{eq:W2:1}
\end{align}
Now we bound $Q(p_{t_{0}})$. We first bound $L(p_{t_{0}})$.
Squaring both sides of the result of Lemma~\ref{lemma:Lconverge} and applying Jensen's inequality now gives
\begin{align}
L(p_{t})\leq 2\exp(-2\alpha^2\lambda_{0}^{2}t)+ 2A_{1}^2\lambda^2\alpha^{-2}\lambda_{0}^{-4}\label{eq:W2:4}.
\end{align}
Plugging $t_{0} = A_{1}^{-1}\alpha^{-1}\lambda^{-1}$ into \eqref{eq:W2:4} gives
\begin{align}
L(p_{t_{0}}) &\leq 2\exp(-2\alpha^{2}\lambda_{0}^{2}t_{0}) + 2A_{1}^2\lambda^2\alpha^{-2}\lambda_{0}^{-4} \notag\\
&= 2\exp\big(-2A_{1}^{-1}\lambda^{-1}\alpha\lambda_{0}^{2}\big) + 2A_{1}^2\lambda^2\alpha^{-2}\lambda_{0}^{-4} \notag\\
&\leq 4A_{1}^2\lambda^2\alpha^{-2}\lambda_{0}^{-4}, \label{eq:W2:2}
\end{align}
where the last inequality is by $\exp(-2z) = [\exp(-z)]^2\leq [1/z]^2$ for any $z>0$. We then bound $D_{\mathrm{KL}}(p_{t_{0}}||p_{0})$. By Lemma \ref{lemma:Dcontroll}, we have
\begin{align}
D_{\mathrm{KL}}(p_{t_{0}}||p_{0}) \leq 2A_{2}^2\alpha^{-2}\lambda_{0}^{-4}+ 2A_{2}^{2}A_{1}^{2}\lambda^{2}\lambda_{0}^{-4}t_{0}^2 
 =4A_{2}^2\alpha^{-2}\lambda_{0}^{-4}.\label{eq:W2:3}
\end{align}
Plugging \eqref{eq:W2:2} and \eqref{eq:W2:3} into \eqref{eq:W2:1} gives
\begin{align*}
D_{\mathrm{KL}}(p_{t}||p_{0}) &\leq \lambda^{-1}Q(p_{t_{0}}) 
= \lambda^{-1}L(p_{t_{0}}) + D_{\mathrm{KL}}(p_{t_{0}}||p_{0})
\leq 4A_{2}^2\alpha^{-2}\lambda_{0}^{-4} + 4A_{1}^2\lambda\alpha^{-2}\lambda_{0}^{-4}.
\end{align*}

This completes the proof.
\end{proof}

\subsection{Proof of Lemma~\ref{lemma:rademacherbound_KL}}
\begin{proof}[Proof of Lemma~\ref{lemma:rademacherbound_KL}]
Our proof is inspired by the Rademacher complexity bound for discrete distributions given by \citet{meir2003generalization}. Let $\gamma$ be a parameter whose value will be determined later in the proof. We have
\begin{align*}
    \mathfrak{R}_n(\cF_{\text{KL}}(M)) &= \frac{\alpha}{\gamma}\cdot \EE_{\bxi} \Bigg[ \sup_{p:D_{\text{KL}}(p || p_0) \leq M} \int_{\RR^{d+1}}\frac{\gamma}{n} \sum_{i=1}^{n} \xi_i uh(\btheta, \xb_i)p(\btheta,u)d\btheta du \Bigg]\\
    &\leq \frac{\alpha}{\gamma}\cdot \Bigg\{ M + \EE_{\bxi} \log \Bigg[ \int \exp\Bigg( \frac{\gamma}{n} \sum_{i=1}^{n} \xi_i uh(\btheta, \xb_i) \Bigg) p_0(\btheta, u) d\btheta du \Bigg] \Bigg\}\\
    &\leq \frac{\alpha}{\gamma}\cdot \Bigg\{ M +  \log \Bigg[ \int \EE_{\bxi} \exp\Bigg( \frac{\gamma}{n} \sum_{i=1}^{n} \xi_i uh(\btheta, \xb_i) \Bigg) p_0(\btheta, u) d\btheta du \Bigg] \Bigg\},
\end{align*}
where the first inequality follows by the Donsker-Varadhan representation of KL-divergence \citep{donsker1983asymptotic}, and the second inequality follows by Jensen's inequality. Note that $\xi_1,\ldots,\xi_n$ are i.i.d. Rademacher random variables. By standard tail bound we have
\begin{align*}
    \EE_{\bxi}\exp\Bigg[ \frac{\gamma}{n} \sum_{i=1}^{n} \xi_i uh(\btheta, \xb_i) \Bigg] \leq \exp\Bigg[ \frac{\gamma^2}{2n^2} \sum_{i=1}^n u^2 h^2(\btheta,\xb_i) \Bigg].
\end{align*}
Therefore
\begin{align*}
    \mathfrak{R}_n(\cF_{\text{KL}}(M)) &\leq \frac{\alpha}{\gamma}\cdot \Bigg\{ M +  \log \Bigg[ \int  \exp\Bigg( \frac{\gamma^2}{2n^2} \sum_{i=1}^n u^2 h^2(\btheta,\xb_i) \Bigg)  p_0(\btheta, u) d\btheta du \Bigg] \Bigg\}.
\end{align*}
Now by the assumption that $h(\btheta,\xb) \leq G$, we have
\begin{align*}
    \int  \exp\Bigg( \frac{\gamma^2}{2n^2} \sum_{i=1}^n u^2 h^2(\btheta,\xb_i) \Bigg)  p_0(\btheta, u) d\btheta du & \leq \int \exp\Bigg( \frac{\gamma^2 G^2}{2n} u^2 \Bigg)  p_0(\btheta, u) d\btheta du\\
    & = \frac{1}{\sqrt{2\pi}} \cdot \sqrt{ \frac{2\pi}{ 1 - \gamma^2 G^2 n^{-1} } }\\
    &= \sqrt{\frac{1}{1 - \gamma^2 G^2 n^{-1} }}.
\end{align*}
Therefore we have
\begin{align*}
    \mathfrak{R}_n(\cF_{\text{KL}}(M))) \leq \frac{\alpha}{\gamma}\cdot \Bigg[ M + \log \Bigg(  \sqrt{\frac{1}{1 - \gamma^2 G^2  n^{-1} }} \Bigg)  \Bigg].
\end{align*}
Setting $\gamma = G^{-1}  \sqrt{M n}$ and applying the inequality $\log(1 - z) \geq - 2z $ for $z\in [0,1/2]$ gives 
\begin{align*}
    \mathfrak{R}_n(\cF_{\text{KL}}(M))) &\leq \frac{G \alpha}{ \sqrt{M n}}\cdot \Bigg[ M +  \log \Bigg(  \sqrt{\frac{1}{1 - M }} \Bigg)  \Bigg] 
    \leq 2 G \alpha \sqrt{\frac{M}{n}}.
\end{align*}
This completes the proof.
\end{proof}

\subsection{Proof of Lemma~\ref{lemma:expected01bound}}
\begin{proof}[Proof of Lemma~\ref{lemma:expected01bound}]
We first introduce the following ramp loss function, which is frequently used in the analysis of generalization bounds \citep{bartlett2017spectrally, li2018tighter} for binary classification problems.
\begin{align*}
    \ell_{\text{ramp}} (y', y) = \left\{ 
    \begin{array}{ll}
        0 & \text{if } y'y \geq 1/2, \\
        -2y'y + 1, & \text{if } 0 \leq y'y < 1/2,\\
        1,& \text{if } y'y < 0.
    \end{array}
    \right.
\end{align*}
Then by definition, we see that $\ell_{\text{ramp}}(y',y)$ is $2$-Lipschitz in the first argument, $\ell_{\text{ramp}}(y,y) = 0$, $|\ell_{\text{ramp}}(y',y)| \leq 1$, and 
\begin{align}\label{eq:ramplossproperty}
    \ell^{\text{0-1}} (y', y) \leq \ell_{\text{ramp}} (y', y) \leq |y' - y|
\end{align}
for all $y'\in \RR$ and $y\in \{\pm 1 \}$. By the Lipschitz and boundedness properties of the ramp loss, we apply the standard properties of Rademacher complexity
\citep{bartlett2002rademacher,mohri2018foundations,shalev2014understanding} and obtain  that with probability at least $1 - \delta$,
\begin{align*}
    \EE_{\cD}[\ell_{\text{ramp}}( f(\xb) , y )/2]
    \leq \EE_{S}[\ell_{\text{ramp}}( f(\xb) , y )/2] + 2 \mathfrak{R}_n(\cF_{\text{KL}}(M))  + 3\sqrt{\frac{\log(2/\delta)}{2n}}
\end{align*}
for all $f\in \cF_{\text{KL}}(M) $. 
Now we have
\begin{align*}
    \EE_\cD [\ell^{\text{0-1}} (f(\xb), y) ] &\leq 2 \EE_\cD [\ell_{\text{ramp}} (f(\xb), y) /2 ] \\
    &\leq \EE_{S}[\ell_{\text{ramp}} ( f(\xb) , y )] + 4 \mathfrak{R}_n(\cF_{\text{KL}}(M)) + 6\sqrt{\frac{\log(2/\delta)}{2n}}\\
    &\leq \EE_{S}[ | f(\xb) - y | ] + 4 \mathfrak{R}_n(\cF_{\text{KL}}(M)) + 6\sqrt{\frac{\log(2/\delta)}{2n}}\\
    &\leq \sqrt{\EE_{S}[ | f(\xb) - y |^2 ]} + 4 \mathfrak{R}_n(\cF_{\text{KL}}(M)) + 6\sqrt{\frac{\log(2/\delta)}{2n}}.
\end{align*}
Here the first and third inequalities follow by the first and second parts of the inequality in \eqref{eq:ramplossproperty} respectively, and the last inequality uses Jensen's inequality. This completes the proof.
\end{proof}

\section{Proof of Lemmas in Appendix~\ref{appendix:proof_part1}}\label{appendix:proof_part2}
In this section we provide the proof of technical lemmas we use in Appendix~\ref{appendix:proof_part1}.

\subsection{Proof of Lemma~\ref{lemma:new2oldassumption}}
Here we provide the proof of Lemma~\ref{lemma:new2oldassumption}, which is essentially based on direct calculations on the activation function and the assumption that $\| \xb \|_2\leq 1$.
\begin{proof}[Proof of Lemma~\ref{lemma:new2oldassumption}]
By $h(\btheta, \xb) = \tilde{h}(\btheta^{\top}\xb)$, we have the following identities.
\begin{align*}
\nabla_{\btheta}h(\btheta, \xb) = \tilde{h}'(\btheta^{\top}\xb)\xb, \ \Delta h(\btheta^{\top}\xb) = \sum_{i=1}\tilde{h}''(\btheta^{\top}\xb)x_{i}^2 = \tilde{h}''(\btheta^{\top}\xb)\|\xb\|_{2}^2, \ \nabla_{\btheta}h(\btheta, \xb)\cdot \btheta = \tilde{h}'(\btheta^{\top}\xb)\btheta^{\top}\xb.
\end{align*}
By $|\tilde{h}(z)| \leq G$ in Assumption~\ref{assmption0} and $\|\xb\|_{2}\leq 1$ in Assumption \ref{assmption:xnorm}, we have
\begin{align*}
|h(\btheta, \xb)| \leq G, 
\end{align*}
which gives the first bound. The other results can be derived similarly, which we present as follows. 
By $|\tilde{h}'(z)|\leq G$ and $\|\xb\|_{2}\leq 1$, we have 
\begin{align*}
\|\nabla_{\btheta}h(\btheta, \xb)\|_{2} = \|\tilde{h}'(\btheta^{\top}\xb)\xb\|_{2} \leq G,
\end{align*}
which gives the second bound. 
By $|\tilde{h}''(z)|\leq G$ and $\|\xb\|_{2}\leq 1$, we have 
\begin{align*}
|\Delta h(\btheta, \xb)| = |\tilde{h}''(\btheta^{\top}\xb)\|\xb\|_{2}^2| \leq G.
\end{align*}
Moreover, based on the same assumptions we also have
\begin{align*}
\|\nabla_{\btheta}h(\btheta_{1}, \xb) - \nabla_{\btheta}h(\btheta_{2}, \xb)\|_{2} &= \|\tilde{h}'(\btheta_{1}^{\top}\xb)\xb - \tilde{h}'(\btheta_{2}^{\top}\xb)\xb\|_{2} \\
&\leq |\tilde{h}'(\btheta_{1}^{\top}\xb) - \tilde{h}'(\btheta_{2}^{\top}\xb)|\\
&\leq G|\btheta_{1}^{\top}\xb - \btheta_{2}^{\top}\xb| \\
&\leq G\|\btheta_{1}^{\top} - \btheta_{2}^{\top}\|_{2}.
\end{align*}
Therefore the third and fourth bounds hold. 
Applying the bound $|\big(z\tilde{h}'(z)\big)'|\leq G$ and $\|\xb\|_{2} \leq 1$ gives the fifth bound:
\begin{align*}
\|\nabla_{\btheta}\big(\nabla_{\btheta}h(\theta, \xb)\cdot \btheta\big)\|_{2} = \|\nabla_{\btheta}\big(\tilde{h}'(\btheta^{\top}\xb)\btheta^{\top}\xb\big)\|_{2} = \|\xb\|_{2}\big|\big(z\tilde{h}'(z)\big)'|_{z = \btheta^{\top}\xb}\big| \leq G.
\end{align*}
Finally, by $|\tilde{h}'''(z)|\leq G$ and $\|\xb\|_{2} \leq 1$, we have
\begin{align*}
\|\nabla_{\btheta} \Delta_{\btheta}h(\btheta, \xb)\|_{2} = \|\nabla_{\btheta} \tilde{h}''(\btheta^{\top}\xb)\|_{2}\|\xb\|_{2}^2 \leq |\tilde{h}'''(\btheta^{\top}\xb)|\|\xb\|_{2}^3 \leq G.
\end{align*}
This completes the proof.
\end{proof}

\subsection{Proof of Lemma~\ref{lemma:H1}}

The following lemma bounds the second moment of a distribution $p$ that is close to $p_0$ in $2$-Wasserstein distance.
\begin{lemma}\label{lemma: Variance}
For $\cW_2(p, p_{0}) \leq \sqrt{d + 1}$, the following bound holds:
\begin{align*}
\EE_{p}(\|\btheta\|_{2}^2 + u^2) &\leq 4d + 4
\end{align*}
\end{lemma}

The following lemma is a reformulation of Lemma C.8 in \citet{xu2018global}. For completeness, we provide its proof in Appendix~\ref{appendix:proof_part2}.

\begin{lemma}\label{lemma: second-order}
For $\cW_2(p, p_{0}) \leq \sqrt{d + 1}$, let $g(u, \btheta):\RR^{d+1} \rightarrow \RR$ be a $C^{1}$ function such that 
\begin{align*}
\sqrt{\nabla_{u}g(u, \btheta)^2 + \|\nabla_{\btheta}g(u, \btheta)\|^2} \leq C_{1}\sqrt{u^2 + \|\btheta\|_{2}^2} + C_{2}, \forall \xb \in \RR^{d'}
\end{align*}
for some constants $C_{1}, C_{2} \geq 0$. Then 
\begin{align*}
\big|\EE_{p}[g(u, \btheta)] - \EE_{p_{0}}[g(u_{0}, \btheta_{0})] \big| \leq \Big(2C_{1} \sqrt{d + 1} + C_{2}\Big)\mathcal{W}_{2}(p, p_{0}).
\end{align*}
\end{lemma}

\begin{proof}[Proof of Lemma~\ref{lemma:H1}]
Let $\pi^*$ be the optimal coupling of $\cW_2(p,p_0)$. Then we have
\begin{align}
\big|\Hb_{1}(p)_{i,j} - \Hb_{1}(p_{0})_{i,j}\big| &= \big| \EE_{\pi^{*}}[u^{2}\nabla_{\btheta}h(\btheta,\xb_{i})\cdot\nabla_{\btheta}h(\btheta,\xb_{j})] -  \EE_{\pi^{*}}[u_{0}^{2}\nabla_{\btheta}h(\btheta_{0},\xb_{i})\cdot\nabla_{\btheta}h(\btheta_{0},\xb_{j})]\big|\notag\\
&\leq \underbrace{ \big| \EE_{\pi^{*}}[(u^2 - u_{0}^2)\nabla_{\btheta}h(\btheta, x_{i})\cdot \nabla_{\btheta}h(\btheta, x_{j})]\big| }_{I_1} \notag \\ 
&\qquad + \underbrace{ \big|\EE_{\pi^{*}}\big[u_{0}^2\big(\nabla_{\btheta}h(\btheta, x_{i})\cdot \nabla_{\btheta}h(\btheta, x_{j}) - \nabla_{\btheta}h(\btheta_{0}, x_{i})\cdot \nabla_{\btheta}h(\btheta_{0}, x_{j})\big)\big]\big|}_{I_2} .
\label{eq:H1:1}
\end{align}
We first bound $I_1$ as follows.
\begin{align}
I_1&\leq G^2 \EE_{\pi^{*}}[|u^2 - u_{0}^2|]\notag \\
&\leq G^2\sqrt{\EE_{\pi^{*}}[(u -u_{0})^2]} \sqrt{\EE_{\pi^{*}}[(u + u_{0})^2]}\notag\\
&\leq G^2 \cW_2(p, p_{0}) \sqrt{2\EE_{p}[u^2] + 2\EE_{p_{0}}[u_{0}^2]}\notag\\
&\leq G^2 \cW_2(p, p_{0}) \sqrt{8d + 10},\label{eq:H1:2}
\end{align}
where the first inequality is by $\|\nabla_{\btheta}h(\btheta, x_{i})\|_{2} \leq G$ in Lemma~\ref{lemma:new2oldassumption}, the second inequality is by Cauchy-Schwarz inequality, the third inequality is by Jensen's inequality and the last inequality is by Lemma \ref{lemma: Variance}.
Next, We bound $I_2$ in \eqref{eq:H1:1}. For any given $r > 0$ we have
\begin{align}
I_2 &\leq \EE_{\pi^{*}}\big[u_{0}^2 \ind(|u_{0} < r|) \big|\nabla_{\btheta}h(\btheta, x_{i})\cdot \nabla_{\btheta}h(\btheta, x_{j}) - \nabla_{\btheta}h(\btheta_{0}, x_{i})\cdot \nabla_{\btheta}h(\btheta_{0}, x_{j})\big|\big]\notag \\
& \qquad + \EE_{\pi^{*}}\big[u_{0}^2 \ind(|u_{0} \geq r|) \big|\nabla_{\btheta}h(\btheta, x_{i})\cdot \nabla_{\btheta}h(\btheta, x_{j}) - \nabla_{\btheta}h(\btheta_{0}, x_{i})\cdot \nabla_{\btheta}h(\btheta_{0}, x_{j})\big|\big]\notag \\
& \leq r^2 \EE_{\pi^{*}}\big[\big|\nabla_{\btheta}h(\btheta, x_{i})\cdot \nabla_{\btheta}h(\btheta, x_{j}) - \nabla_{\btheta}h(\btheta_{0}, x_{i})\cdot \nabla_{\btheta}h(\btheta_{0}, x_{j})\big|\big] + 2G^2 \EE_{\pi^{*}}[u_{0}^2\ind(|u_{0}\geq r|)],\label{eq:H1:3}
\end{align}
where the second inequality is by $\|\nabla_{\btheta}h(\btheta, x_{i})\|_{2} \leq G$ Lemma~\ref{lemma:new2oldassumption}. We further bound the first term on the right-hand side of \eqref{eq:H1:3},
\begin{align}
&\EE_{\pi^{*}}\big[\big|\nabla_{\btheta}h(\btheta, x_{i})\cdot \nabla_{\btheta}h(\btheta, x_{j}) - \nabla_{\btheta}h(\btheta_{0}, x_{i})\cdot \nabla_{\btheta}h(\btheta_{0}, x_{j})\big|\big]\notag\\
&\qquad\leq\EE_{\pi^{*}}\big[\big|\nabla_{\btheta}h(\btheta, x_{i})\cdot \big(\nabla_{\btheta}h(\btheta, x_{j})- \nabla_{\btheta}h(\btheta_{0}, x_{j})\big)\big|\big]\notag\\
&\qquad +  \EE_{\pi^{*}}\big[\big|\nabla_{\btheta}h(\btheta_{0}, x_{j})\cdot \big(\nabla_{\btheta}h(\btheta, x_{i})- \nabla_{\btheta}h(\btheta_{0}, x_{i})\big)\big|\big]\notag\\
&\qquad\leq 2G^{2}\cW_2(p, p_{0}),\label{eq:H1:4}
\end{align}
where the last inequality is by  $\|\nabla_{\btheta}h(\btheta, \xb)\|_{2}\leq G$
and $\|\nabla_{\btheta}h(\btheta, \xb) - \nabla_{\btheta}h(\btheta_{0}, \xb)\|_{2} \leq G\|\btheta - \btheta_{0}\|_{2}$ in Lemma~\ref{lemma:new2oldassumption}.
Plugging \eqref{eq:H1:4} into \eqref{eq:H1:3} yields
\begin{align}\label{eq:H1:5}
    I_2 \leq 2 r^2 G^{2}\cW_2(p, p_{0}) + 2G^2 \EE_{\pi^{*}}[u_{0}^2\ind(|u_{0}\geq r|)].
\end{align}
Further plugging \eqref{eq:H1:2} and \eqref{eq:H1:5} into \eqref{eq:H1:1}, we obtain
\begin{align*}
\big|\Hb_{1}(p)_{i,j} - \Hb_{1}(p_{0})_{i,j}\big| &\leq G^2 \cW_2(p, p_{0})\sqrt{8d + 10} + 2r^2G^{2}\cW_2(p, p_{0})\\
&\qquad + 2G^2 \EE_{p_{0}}[u_{0}^2\ind(|u_{0}\geq r|)].
\end{align*}
This finishes the proof.
\end{proof}

\subsection{Proof of Lemma~\ref{lemma:H2}}
Here we provide the proof of Lemma~\ref{lemma:H2}, which is essentially based on a direct application of Lemma~\ref{lemma:new2oldassumption} and the definition of $2$-Wasserstein distance.

\begin{proof}[Proof of Lemma~\ref{lemma:H2}]
Denote $\hat{H}_{i,j}(\btheta, u) = h(\btheta,\xb_{i})h(\btheta,\xb_{j})$, then we have $\Hb_{2}(p)_{i,j} = \EE_{p}[\hat{H}_{i,j}(\btheta, u)]$.
Calculating the gradient of $\hat{H}_{i,j}(\btheta, u)$, we have
\begin{align*}
\nabla_{u} \hat{H}_{i,j}(\btheta, u) = 0,\qquad \|\nabla_{\btheta} \hat{H}_{i,j}(\btheta, u)\|_{2} \leq 2\|\nabla_{\btheta}h(\btheta,\xb_{i})\|_{2}|h(\btheta,\xb_{j})| \leq 2G^{2},
\end{align*}
where the second inequality is by Lemma~\ref{lemma:new2oldassumption} .
Applying Lemma \ref{lemma: second-order} gives
\begin{align*}
|\Hb_{2}(p)_{i,j} -  \Hb_{2}(p_{0})_{i,j}| \leq  2G^{2}\cW_2(p, p_{0}).
\end{align*}
This finializes our proof.
\end{proof}

\subsection{Proof of Lemma~\ref{lemma: tail bound}}
Lemma~\ref{lemma: tail bound} gives a tail bound on $p_0$, which is essentially a basic property of Gaussian distribution. For completeness we present the detailed proof as follows.

\begin{proof}[Proof of Lemma~\ref{lemma: tail bound}]
By the definition of $p_0$ we have
\begin{align*}
\EE_{p_{0}}[u_{0}^2\ind(|u_{0}| \geq r)] &= \frac{2}{\sqrt{2\pi}}\int_{r}^{\infty}u_{0}^2\exp(-u_{0}^2/2)du_{0}
=\frac{2}{\sqrt{\pi}}\int_{r^2/2}^{\infty}t^{1/2}\exp(-t)dt
\end{align*}
Now by the fact that $4z/\pi \leq \exp(z), \forall z\in \RR$, we have
\begin{align*}
\EE_{p_{0}}[u_{0}^2\ind(|u_{0} \geq r|)] \leq \int_{r^2/2}^{\infty}\exp(-t/2)dt =\frac{1}{2}\exp\bigg(-\frac{r^2}{4}\bigg),
\end{align*}
which finalizes our proof.
\end{proof}

\subsection{Proof of Lemma~\ref{lemma:Ldynamic}}

We first introduce some notations on the first variations. For $i\in[n]$, $\frac{\partial \fb(t)_{i}}{\partial p_{t}}$,  $\frac{\partial L(p_{t})}{\partial p_{t}}$, $\frac{\partial  D_{\mathrm{KL}}(p_{t}||p_{0})}{\partial p_{t}}$ and $\frac{\partial Q(p_{t})}{\partial p_{t}}$ are defined as follows.
\begin{align}
 \frac{\partial \fb(t)_{i}}{\partial p_{t}} :&= \alpha u h(\btheta, \xb_{i}),\label{eq:def_firstvariation_f}\\
\frac{\partial L(p_{t})}{\partial p_{t}} :&= \EE_{S}\big[\nabla_{y'} \phi\big(f(p_{t},\xb), y\big)\cdot \alpha uh(\btheta, \xb)\big], \label{eq:def_firstvariation_L} \\
\frac{\partial  D_{\mathrm{KL}}(p_{t}||p_{0})}{\partial p_{t}} :&= \log(p_{t}/p_{0}) + 1, \label{eq:def_firstvariation_D} \\
\frac{\partial Q(p_{t})}{\partial p_{t}} :&=\frac{\partial L(p_{t})}{\partial p_{t}} + \lambda \frac{\partial  D_{\mathrm{KL}}(p_{t}||p_{0})}{\partial p_{t}}\notag\\
&= \EE_{S}\big[\nabla_{y'} \phi\big(f(p_{t},\xb), y\big)\cdot \alpha uh(\btheta, \xb) + \lambda \log(p_{t}/p_{0}) + \lambda\big]. \label{eq:def_firstvariation_Q}
\end{align}

The following lemma summarizes some direct calculations on the relation between these first variations defined above and the time derivatives of $\fb(t)_{i}$, $L(p_t)$, $D_{\text{KL}}(p_t||p_0)$ and $Q(p_t)$. Note that these results are well-known results in literature, but for completeness we present the detailed calculations in Appendix~\ref{section:proof_of_lemma:chainrule}. 
\begin{lemma}\label{lemma:chainrule}
Let $\frac{\partial \fb(t)_{i}}{\partial p_{t}}$, $\frac{\partial L(p_{t})}{\partial p_{t}}$, $\frac{\partial  D_{\mathrm{KL}}(p_{t}||p_{0})}{\partial p_{t}}$, $\frac{\partial Q(p_{t})}{\partial p_{t}}$ be the first variations defined in
\eqref{eq:def_firstvariation_f}, \eqref{eq:def_firstvariation_L}, \eqref{eq:def_firstvariation_D} and \eqref{eq:def_firstvariation_Q}. Then 
\begin{align*}
\frac{\partial[\fb(t)_{i}-y_{i}]}{\partial t} &= \int_{\RR^{d+1}}\frac{\partial \fb(t)_{i}}{\partial p_{t}}\frac{dp_{t}}{dt}d\btheta du,\\
\frac{\partial L(p_{t})}{\partial t} &= \int_{\RR^{d+1}}\frac{\partial L(p_{t})}{\partial p_{t}}\frac{dp_{t}}{dt}d\btheta du,\\
\frac{\partial  D_{\mathrm{KL}}(p_{t}||p_{0})}{\partial t} &= \int_{\RR^{d+1}}\frac{\partial  D_{\mathrm{KL}}(p_{t}||p_{0})}{\partial p_{t}}\frac{dp_{t}}{dt}d\btheta du,\\
\frac{\partial Q(p_{t})}{\partial t} &= \int_{\RR^{d+1}}\frac{\partial Q(p_{t})}{\partial p_{t}}\frac{dp_{t}}{dt}d\btheta du.
\end{align*}
\end{lemma}

The following lemma summarizes the calculation of the gradients of the first variations defined in \eqref{eq:def_firstvariation_L}, \eqref{eq:def_firstvariation_D} and \eqref{eq:def_firstvariation_Q}.
\begin{lemma}\label{lemma: gradient}
Let $\frac{\partial L(p_{t})}{\partial p_{t}}$, $\frac{\partial  D_{\mathrm{KL}}(p_{t}||p_{0})}{\partial p_{t}}$ and $\frac{\partial Q(p_{t})}{\partial p_{t}}$ be the first variations defined in  \eqref{eq:def_firstvariation_L}, \eqref{eq:def_firstvariation_D} and \eqref{eq:def_firstvariation_Q}. Then their gradients with respect to $u$ and $\btheta$ are given as follows:
\begin{align*}
&\nabla_{u} \frac{\partial L(p_{t})}{\partial p_{t}} =  -\hat{g}_{1}(t, \btheta,u),
\nabla_{\btheta} \frac{\partial L(p_{t})}{\partial p_{t}} =  -\hat{g}_{2}(t, \btheta,u),\\
&\nabla_{u} \frac{\partial D_{\text{KL}}(p_{t}||p_{0})}{\partial p_{t}} = u + \nabla_{\btheta}\log(p_{t}),
\nabla_{\btheta} \frac{\partial  D_{\mathrm{KL}}(p_{t}||p_{0})}{\partial p_{t}} = \btheta + \nabla_{\btheta}\log(p_{t}),\\
&\nabla \frac{\partial Q(p_{t})}{\partial p_{t}} = \nabla \frac{\partial L(p_{t})}{\partial p_{t}} + \lambda \nabla \frac{\partial D_{\text{KL}}(p_{t}||p_{0})}{\partial p_{t}}.
\end{align*}
Moreover, the PDE \eqref{eq:pde_p_t_1} can be written as
\begin{align*}
\frac{d p_{t}}{dt} = \nabla \cdot \bigg[p_{t}(\btheta, u)\nabla \frac{\partial Q(p_{t})}{\partial p_{t}}\bigg].
\end{align*}
\end{lemma}

\begin{proof}[Proof of Lemma~\ref{lemma:Ldynamic}]
By Lemma \ref{lemma:chainrule}, we have the following chain rule
\begin{align}
\frac{\partial L(p_{t})}{\partial t} &= \int_{\RR^{d+1}}\frac{\partial L(p_{t})}{\partial p_{t}}\frac{dp_{t}}{dt}d\btheta du \notag\\
&=\int_{\RR^{d+1}}\frac{\partial L(p_{t})}{\partial p_{t}}\nabla \cdot \bigg[p_{t}(\btheta, u)\nabla \frac{\partial Q(p_{t})}{\partial p_{t}}\bigg]d\btheta du \notag\\
&=-\int_{\RR^{d+1}}p_{t}(\btheta, u)\bigg[\nabla\frac{\partial L(p_{t})}{\partial p_{t}}\bigg] \cdot \bigg[\nabla \frac{\partial Q(p_{t})}{\partial p_{t}}\bigg]d\btheta du \notag\\
&=-\underbrace{\int_{\RR^{d+1}}p_{t}(\btheta, u)\bigg[\nabla\frac{\partial L(p_{t})}{\partial p_{t}}\bigg] \cdot \bigg[\nabla \frac{\partial L(p_{t})}{\partial p_{t}}\bigg]d\btheta du}_{I_1} \notag\\
&\qquad - \underbrace{ \lambda\int_{\RR^{d+1}}p_{t}(\btheta, u)\bigg[\nabla\frac{\partial L(p_{t})}{\partial p_{t}}\bigg] \cdot \bigg[\nabla \frac{\partial D_{}(p_{t}||p_{0})}{\partial p_{t}}\bigg]d\btheta du}_{I_2}, \label{eq:1.1}
\end{align}
where the second and last equation is by Lemma \ref{lemma: gradient}, the third inequality is by apply integration by parts.
We now proceed to calculate $I_1$ and $I_2$ based on the calculations of derivatives in Lemma \ref{lemma: gradient}. For $I_1$, we have
\begin{align}
I_1 = \int_{\RR^{d+1}}p_{t}(\btheta, u) \|\hat{g}_{1}(t, \btheta, u)\|_2^2 d\btheta du + \int_{\RR^{d+1}}p_{t}(\btheta, u) |\hat{g}_{2}(t, \btheta, u)|^2 d\btheta d\btheta du. \label{eq:1.2}
\end{align}
Similarly, for $I_2$, we have
\begin{align}
I_2 &= \int_{\RR^{d+1}}p_{t}(\btheta, u) [-\hat{g}_{1}(t, \btheta, u)] \cdot [u + \nabla_{u}\log(p_{t})] d\btheta du \notag \\
&\qquad + \int_{\RR^{d+1}}p_{t}(\btheta, u) [-\hat{g}_{2}(t, \btheta, u)] \cdot [ \btheta + \nabla_{\btheta}\log(p_{t})] d\btheta du \notag\\
&= -\int_{\RR^{d+1}}p_{t}(\btheta, u) [\hat{g}_{1}(t, \btheta, u) \cdot u + \hat{g}_{2}(t, \btheta, u) \btheta] d\btheta du \notag \\
&\qquad - \int_{\RR^{d+1}}[\hat{g}_{1}(t, \btheta, u)\cdot \nabla_{u}p_{t}(t, \btheta, u) + \hat{g}_{2}(t, \btheta, u)\cdot\nabla_{\btheta}p_{t}(t, \btheta, u)]d\btheta du \notag\\
&= -\int_{\RR^{d+1}}p_{t}(\btheta, u) [\hat{g}_{1}(t, \btheta, u) \cdot u + \hat{g}_{2}(t, \btheta, u) \btheta] d\btheta du \notag \\
&\qquad +\int_{\RR^{d+1}}p_{t}(t, \btheta, u)[\nabla_{u}\cdot\hat{g}_{1}(t, \btheta, u) + \nabla_{\btheta}\cdot\hat{g}_{2}(t, \btheta, u)]d\btheta du,\label{eq:1.3}
\end{align}
where the second equation is by $p_{t}\nabla\log(p_{t}) = \nabla p_{t}$ and the third equation is by applying integration by parts. Plugging \eqref{eq:1.2} and \eqref{eq:1.3} into \eqref{eq:1.1}, we get
\begin{align*}
\frac{\partial L(p_{t})}{\partial t}  &= -\int_{\RR^{d+1}}p_{t}(\btheta, u) \|\hat{g}_{1}(t, \btheta, u)\|_2^2 d\btheta du - \int_{\RR^{d+1}}p_{t}(\btheta, u) |\hat{g}_{2}(t, \btheta, u)|^2 d\btheta d\btheta du  \notag\\
&\qquad + \lambda\int_{\RR^{d+1}}p_{t}(\btheta, u)[\hat{g}_{1}\cdot u + \hat{g}_{2}\cdot \btheta - \nabla_{u} \cdot \hat{g}_{1} - \nabla_{\btheta} \cdot \hat{g}_{2} ] d\btheta du.
\end{align*}
This completes the proof.
\end{proof}

\subsection{Proof of Lemma~\ref{lemma: Residual Dynamic}}

Here we prove Lemma~\ref{lemma: Residual Dynamic}, which is based on its connection to the Gram matrix of neural tangent kernel.

\begin{proof}[Proof of Lemma~\ref{lemma: Residual Dynamic}]
We first remind the readers of the definitions of the Gram matrices in \eqref{eq:def_H}. 
Let $\bbb(p_{t}) = (f(p_{t},\xb_{1}) - y_{1},\ldots,f(p_{t},\xb_{n}) - y_{n})^\top \in \RR^{n}$. Then by the definitions of $\Hb_{1}(p_{t})$ and $\Hb_{2}(p_{t})$ in \eqref{eq:def_H}, we have 
\begin{align*}
\int_{\RR^{d+1}}p_{t}(\btheta,u)\big[|\EE_{S}[(f(p_{t},\xb)- y)h(\btheta,\xb)]|^{2}d\btheta du &= 
 \frac{1}{n^2}\bbb(p_{t})^{\top}\Hb_{1}(p_{t})\bbb(p_{t}),\\
\int_{\RR^{d+1}}p_{t}(\btheta,u)\big[\|\EE_{S}[(f(p_{t},\xb)- y)u\nabla_{\btheta}h(\btheta,\xb)]\|_2^{2}\big]d\btheta du &= 
 \frac{1}{n^2}\bbb(p_{t})^{\top}\Hb_{2}(p_{t})\bbb(p_{t}).
\end{align*}
Therefore by \eqref{eq:def_H} we have
\begin{align}
    &\int_{\RR^{d+1}}p_{t}(\btheta,u)\big[|\EE_{S}[(f(p_{t},\xb)- y)h(\btheta,\xb)]|^{2} + \|\EE_{S}[(f(p_{t},\xb)- y)u\nabla_{\btheta}h(\btheta,\xb)]\|_2^{2}\big]d\btheta du\notag\\
    &\qquad = \frac{1}{n^2}\bbb(p_{t})^{\top}\Hb(p_{t})\bbb(p_{t})\label{eq:lemmaResidualDynamic_eq1}.
\end{align}
By the definition of $t^*$, for $t \leq t^*$ we have $\cW_2(p_t,p_0) \leq R$, and therefore applying Lemma~\ref{lemma:H} gives
\begin{align}\label{eq:lemmaResidualDynamic_eq2}
    \frac{1}{n^2}\bbb(p_{t})^{\top}\Hb(p_{t})\bbb(p_{t}) \geq \frac{\Lambda\|\bbb(p_{t})\|_{2}^2}{2n^{2}} = \frac{\lambda_{0}^2}{2}L(p_{t}),
\end{align}
where the equation follows by the definition of $\bbb(p_{t})$. 
Plugging \eqref{eq:lemmaResidualDynamic_eq2} into \eqref{eq:lemmaResidualDynamic_eq1} completes the proof.
\end{proof}

\subsection{Proof of Lemma~\ref{lemma:diffusion term1}}

\begin{lemma}\label{lemma: I}
Under Assumptions \ref{assmption:xnorm} and \ref{assmption0}, for all $\cW(p, p_{0}) \leq \sqrt{d+1}$ and $\xb$ the following inequality holds.
\begin{align*}
\big|\EE_{p}\big[ uh(\btheta,\xb) + u\nabla h(\btheta,\xb)\cdot \btheta - u\Delta h(\btheta, \xb) \big]\big| \leq A_{1},
\end{align*}
where $A_{1}$ is defined in Theorem \ref{thm:main}.
\end{lemma}

The proof of Lemma~\ref{lemma:diffusion term1} is based on direct applications of Lemma~\ref{lemma: I}. We present the proof as follows.

\begin{proof}[Proof of Lemma~\ref{lemma:diffusion term1}]
We have the following identities:
\begin{align*}
\hat{g}_{1}(t, \btheta, u)  &=  -\EE_{S}[\nabla_{f}\phi(f(p_{t},\xb),y)\alpha h(\btheta,\xb)],\\
\hat{g}_{2}(t, \btheta, u) &= -\EE_{S}[\nabla_{f}\phi(f(p_{t},\xb),y)\alpha u\nabla_{\btheta}h(\btheta,\xb)],\\
\nabla_{u}\cdot\hat{g}_{1}(t, \btheta, u) &= 0,\\
\nabla_{\btheta}\cdot\hat{g}_{2}(t, \btheta, u) &=  -\EE_{S}[\nabla_{f}\phi(f(p_{t},\xb),y)\alpha u\Delta h(\btheta,\xb)].
\end{align*}
Base on these identities we can derive
\begin{align*}
&\bigg|\int_{\RR^{d+1}}p_{t}(\btheta, u)[\hat{g}_{1}\cdot u + \hat{g}_{2}\cdot \btheta - \nabla_{u} \cdot \hat{g}_{1} - \nabla_{\btheta} \cdot \hat{g}_{2} ] d\btheta du \bigg| \\
&\qquad = \bigg|\alpha\EE_{S}\bigg[\nabla_{f}\phi(f(p_{t},\xb),y)\EE_{p_{t}}\Big[\big(u_{t}h(\btheta_{t},\xb_{i}) + u_{t}\nabla h(\btheta_{t},\xb_{i})\cdot \btheta_{t} - u_{t}\Delta h(\btheta_{t}, \xb_{i}) \big)\Big]\bigg]\bigg|\\
&\qquad \leq 2\alpha A_{1}\EE_{S}[|f(p_{t}, \xb) - y|]\\
&\qquad \leq 2\alpha A_{1}\sqrt{L(p_{t})},
\end{align*}
where the first inequality is by Lemma \ref{lemma: I}, the second inequality is by Jensen's inequality.
\end{proof}

\subsection{Proof of Lemma~\ref{lemma:Dcontroll}}

The following lemma summarizes the calculation on the time derivative of $D_{\mathrm{KL}}(p_{t}||p_{0})$. 

\begin{lemma}\label{lemma:Ddynamic}
Let $p_t$ be the solution of PDE \eqref{eq:pde_p_t_1}. Then the following identity holds.
\begin{align*}
\frac{\partial  D_{\mathrm{KL}}(p_{t}||p_{0})}{\partial t} &= -\lambda\int_{\RR^{d+1}}p_{t}(\btheta,u)\|\btheta +  \nabla_{\btheta}\log(p_{t})\|_2^2 - \lambda\int_{\RR^{d+1}}p_{t}(\btheta,u)|u +  \nabla_{u}\log(p_{t})|^2 \notag\\
&\qquad + \int_{\RR^{d+1}}p_{t}(\btheta, u)[\hat{g}_{1}\cdot u + \hat{g}_{2}\cdot \btheta - \nabla_{u} \cdot \hat{g}_{1} - \nabla_{\btheta} \cdot \hat{g}_{2} ] d\btheta du.
\end{align*}
\end{lemma}

In the calculation given by Lemma~\ref{lemma:Ddynamic}, we can see that the (potentially) positive term in $\frac{\partial  D_{\mathrm{KL}}(p_{t}||p_{0})}{\partial t}$ naturally coincides with the corresponding term in $\frac{\partial L(p_{t})}{\partial t}$ given by Lemma~\ref{lemma:Ldynamic}, and a bound of it has already been given in Lemma~\ref{lemma:diffusion term1}. However, for the analysis of the KL-divergence term, we present the following new bound, which eventually leads to a sharper result.

\begin{lemma}\label{lemma:diffusion term2}
Under Assumptions \ref{assmption:xnorm} and \ref{assmption0}, let $A_{2}$ be defined in Theorem \ref{thm:main}. Then for $t \leq t^{*}$,it holds that
\begin{align*}
\int_{\RR^{d+1}}p(\btheta, u)[\hat{g}_{1}\cdot u + \hat{g}_{2}\cdot \btheta - \nabla_{u} \cdot \hat{g}_{1} - \nabla_{\btheta} \cdot \hat{g}_{2} ] d\btheta du \leq 2\alpha  A_{2}\sqrt{L(p_{t})}\sqrt{ D_{\mathrm{KL}}(p_{t}||p_{0})}.
\end{align*}
\end{lemma}

\begin{proof}[Proof of Lemma~\ref{lemma:Dcontroll}]
By Lemma \ref{lemma:Ddynamic},
\begin{align}
\frac{\partial  D_{\mathrm{KL}}(p_{t}||p_{0})}{\partial t} &= -\lambda\int_{\RR^{d+1}}p_{t}(\btheta,u)\|\btheta +  \nabla_{\btheta}\log(p_{t})\|_2^2 - \lambda\int_{\RR^{d+1}}p_{t}(\btheta,u)|u +  \nabla_{u}\log(p_{t})|^2 \notag\\
&\qquad + \int_{\RR^{d+1}}p_{t}(\btheta, u)[\hat{g}_{1}\cdot u + \hat{g}_{2}\cdot \btheta - \nabla_{u} \cdot \hat{g}_{1} - \nabla_{\btheta} \cdot \hat{g}_{2} ] d\btheta du \notag\\
&\leq 2A_{2}\alpha\sqrt{L(p_{t})}\sqrt{ D_{\mathrm{KL}}(p_{t}||p_{0})}, \label{eq:Dcontroll:1}
\end{align}
where the inequality is by Lemma \ref{lemma:diffusion term2}. Notice that $\sqrt{D_{\mathrm{KL}}(p_{0}||p_{0})} = 0$, $\sqrt{D_{\mathrm{KL}}(p_{t}||p_{0})}$ is differentiable at $\sqrt{D_{\mathrm{KL}}(p_{t}||p_{0})} \not= 0$ and from \eqref{eq:Dcontroll:1} the
derivative 
\begin{align*}
\frac{\partial\sqrt{D_{\mathrm{KL}}(p_{t}||p_{0}})}{\partial t} = \frac{\partial D_{\mathrm{KL}}(p_{t}||p_{0}))}{\partial t}\frac{1}{2\sqrt{D_{\mathrm{KL}}(p_{t}||p_{0}})}\leq A_{2}\alpha\sqrt{L(p_{t})},
\end{align*}
which implies
\begin{align*}
\sqrt{ D_{\mathrm{KL}}(p_{t}||p_{0})} &\leq \int_{0}^{t} A_{2}\alpha\sqrt{L(p_{s})}ds \\
&\leq A_{2}\alpha \int_{0}^{t}\exp(-\alpha^2\lambda_{0}^2s) + A_{1}\lambda\alpha^{-1}\lambda_{0}^{-2} ds\\
&\leq A_{2}\alpha^{-1}\lambda_{0}^{-2} + A_{2}A_{1}\lambda\lambda_{0}^{-2}t,
\end{align*}
where the second inequality holds due to Lemma \ref{lemma:Lconverge}. Squaring both sides and applying Jensen's inequality now gives
\begin{align*}
D_{\mathrm{KL}}(p_{t}||p_{0}) &\leq 2A_{2}^2\alpha^{-2}\lambda_{0}^{-4} + 2A_{2}^2 A_{1}^2\lambda^2\lambda_{0}^{-4}t^2.
\end{align*}
This completes the proof.
\end{proof}

\subsection{Proof of Lemma~\ref{lemma:Qdynamic}}

\begin{proof}[Proof of Lemma~\ref{lemma:Qdynamic}]
By Lemma \ref{lemma:chainrule}, we get
\begin{align*}
\frac{\partial Q(p_{t})}{\partial t} &= \int_{\RR^{d+1}}\frac{\partial Q(p_{t})}{\partial p_{t}}\frac{dp_{t}}{dt}d\btheta du \notag\\
&=\int_{\RR^{d+1}}\frac{\partial Q(p_{t})}{\partial p_{t}}\nabla \cdot \bigg[p_{t}(\btheta, u)\nabla \frac{\partial Q(p_{t})}{\partial p_{t}}\bigg]d\btheta du \notag\\
&=-\int_{\RR^{d+1}}p_{t}(\btheta, u)\bigg\|\nabla\frac{\partial Q(p_{t})}{\partial p_{t}}\bigg\|_{2}^{2}d\btheta du\\
&=-\int_{\RR^{d+1}}p_{t}(\btheta,u)\|\hat{g}_{2} - \lambda\btheta -  \lambda\nabla_{\btheta}\log(p_{t})\|_2^2 - \int_{\RR^{d+1}}p_{t}(\btheta,u)|\hat{g}_{1} - \lambda u -  \lambda\nabla_{u}\log(p_{t})|^2\\
&\leq 0,
\end{align*}
where the third equation is by applying integration by parts and the fourth equation is by Lemma \ref{lemma: gradient}.
\end{proof}

\section{Proof of Auxiliary Lemmas in Appendix~\ref{appendix:proof_part2}}\label{appendix:proof_part3}
\subsection{Proof of Lemma~\ref{lemma: Variance}}
\begin{proof}[Proof of Lemma~\ref{lemma: Variance}]
Let $\pi^{*}(p_{0},p)$ be the coupling that achieves the $2$-Wasserstein distance between $p_{0}$ and $p$. Then by definition, 
\begin{align*}
\EE_{\pi^{*}}(\|\btheta\|_{2}^2 + u^2) &\leq  \EE_{\pi^{*}}(2\|\btheta - \btheta_{0}\|_{2}^2 + 2\|\btheta_{0}\|_{2}^2 + 2(u-u_{0})^2 + 2u_{0}^2) \\
&\leq 2R^2 + 2d + 2\\
&\leq 4d + 4,
\end{align*}
where the last inequality is by the assumption that $\mathcal{W}_{2}(p, p_{0}) \leq \sqrt{d + 1}$. This finishes the proof.
\end{proof}

\subsection{Proof of Lemma~\ref{lemma: second-order}}
\begin{proof}[Proof of Lemma~\ref{lemma: second-order}]
By Lemma C.8 in \citet{xu2018global}, we have that 
\begin{align*}
\bigg|\EE_{p}[g(u, \btheta)] - \EE_{p_{0}}[g(u_{0}, \btheta_{0})] \bigg| \leq (C_{1}\sigma + C_{2})\mathcal{W}_{2}(p, p_{0}),
\end{align*}
where $\sigma^{2} = \max\{\EE_{p}[u^2 + \btheta^2], \EE_{p_{0}}[u_{0}^2 + \btheta_{0}^2]\}$.
Then by Lemma \ref{lemma: Variance}, we get $\sigma \leq 2\sqrt{d+1}$. Substituting the upper bound of $\sigma$ into the above inequality completes the proof.
\end{proof}

\subsection{Proof of Lemma~\ref{lemma:chainrule}}
\label{section:proof_of_lemma:chainrule}

\begin{proof}[Proof of Lemma~\ref{lemma:chainrule}] By chain rule and the definition of $\fb(t)$, we have
\begin{align*}
\frac{\partial[\fb(t)_{i}-y_{i}]}{\partial t} &= \frac{d}{dt} \int_{\RR^{d+1}}\alpha uh(\btheta, \xb_{i})p_{t}(\btheta,u)d\btheta du\\
&= \int_{\RR^{d+1}}\alpha uh(\btheta, \xb_{i})\frac{dp_{t}}{dt}(\btheta,u)d\btheta du\\
&=\int_{\RR^{d+1}}\frac{\partial \fb(t)_{i}}{\partial p_{t}}\frac{dp_{t}}{dt}d\btheta du,
\end{align*}
where the last equation follows by the definition of the first variation $\frac{\partial \fb(t)_{i}}{\partial p_{t}}$. This proves the first identity. Now we bound the second identity,
\begin{align*}
\frac{\partial L(p_{t})}{\partial t} &= \EE_{S}\bigg[\nabla_{y'} \phi\big(f(p_{t},\xb), y\big)\frac{d}{dt}f(p_{t},\xb)\bigg]\\
&= \EE_{S}\bigg[\nabla_{y'} \phi\big(f(p_{t},\xb), y\big)\frac{d}{dt} \int_{\RR^{d+1}}\alpha uh(\btheta, \xb)p_{t}(\btheta,u)d\btheta du\bigg] \\
&= \EE_{S}\bigg[\nabla_{y'} \phi\big(f(p_{t},\xb), y\big) \int_{\RR^{d+1}}\alpha uh(\btheta, \xb)\frac{dp_{t}(\btheta,u)}{dt}d\btheta du\bigg]\\
&= \int_{\RR^{d+1}}\frac{\partial L(p_{t})}{\partial p_{t}}\frac{dp_{t}}{dt}d\btheta du,
\end{align*}
where the last equation follows by the definition of the first variation $\frac{\partial L(p_{t})}{\partial p_{t}}$. This proves the second identity. Similarly, for $\frac{\partial  D_{\mathrm{KL}}(p_{t}||p_{0})}{\partial t}$, we have
\begin{align*}
\frac{\partial D_{\text{KL}}(p_{t}||p_{0})}{\partial t} &=  \frac{d}{dt}\int p_{t}\log(p_{t}/p_{0})d\btheta du
= \int \frac{dp_{t}}{dt}\log(p_{t}/p_{0}) + \frac{dp_{t}}{dt}d\btheta du
= \int_{\RR^{d+1}}\frac{\partial  D_{\mathrm{KL}}(p_{t}||p_{0})}{\partial p_{t}}\frac{dp_{t}}{dt}d\btheta du.
\end{align*}
Notice that $Q(p_{t}) = L(p_{t})+\lambda  D_{\mathrm{KL}}(p_{t}||p_{0})$, so we have
\begin{align*}
\frac{\partial Q(p_{t})}{\partial t} &= \frac{\partial L(p_{t})}{\partial t} + \lambda \frac{\partial  D_{\mathrm{KL}}(p_{t}||p_{0})}{\partial t} \\
&= \int_{\RR^{d+1}}\frac{\partial L(p_{t})}{\partial p_{t}}\frac{dp_{t}}{dt}d\btheta du + \lambda  \int_{\RR^{d+1}}\frac{\partial  D_{\mathrm{KL}}(p_{t}||p_{0})}{\partial p_{t}}\frac{dp_{t}}{dt}d\btheta du \\
&= \int_{\RR^{d+1}}\frac{\partial Q(p_{t})}{\partial p_{t}}\frac{dp_{t}}{dt}d\btheta du,
\end{align*}
where the last equation is by the definition $\frac{\partial Q(p_{t})}{\partial p_{t}} = \frac{\partial L(p_{t})}{\partial p_{t}} + \lambda \frac{\partial  D_{\mathrm{KL}}(p_{t}||p_{0})}{\partial p{t}}$. This completes the proof.
\end{proof}

\subsection{Proof of Lemma~\ref{lemma: gradient}}

\begin{proof}[Proof of Lemma~\ref{lemma: gradient}]
By Lemma \ref{lemma:chainrule},we have
\begin{align*}
&\nabla_{u} \frac{\partial L}{\partial p_{t}} = \nabla_{u}\EE_{S}\big[\nabla_{y'} \phi\big(f(p_{t},\xb), y\big) \alpha uh(\btheta, \xb)\big] = -\hat{g}_{1}(t, \btheta,u),\\
&\nabla_{\btheta} \frac{\partial L}{\partial p_{t}} = \nabla_{\btheta}\EE_{S}\big[\nabla_{y'} \phi\big(f(p_{t},\xb), y\big) \alpha uh(\btheta, \xb)\big] = -\hat{g}_{2}(t, \btheta,u),\\
&\nabla_{u} \frac{\partial D_{kL}(p_{t}||p_{0})}{\partial p_{t}} = \nabla_{u}(\log(p_{t}/p_{0}) + 1) = u + \nabla_{u}\log(p_{t}),\\
&\nabla_{\btheta} \frac{\partial D_{kL}(p_{t}||p_{0})}{\partial p_{t}} = \nabla_{\btheta}(\log(p_{t}/p_{0}) + 1) = \btheta + \nabla_{\btheta}\log(p_{t}).
\end{align*}
This proves the first four identities. For the last one, 
by the definition 
\begin{align*}
\nabla \frac{\partial Q(p_{t})}{\partial p_{t}} = \nabla \frac{\partial L(p_{t})}{\partial p_{t}} + \lambda \nabla\frac{\partial D_{\text{KL}}(p_{t}||p_{0})}{\partial p_{t}},
\end{align*}
we have
\begin{align*}
\nabla \cdot \bigg[p_{t}(\btheta, u)\nabla \frac{\partial Q(p_{t})}{\partial p_{t}}\bigg] &= \nabla \cdot \bigg[p_{t}(\btheta, u)\nabla \frac{\partial L}{\partial p_{t}}\bigg] + \lambda \nabla \cdot \bigg[p_{t}(\btheta, u)\nabla \frac{\partial  D_{\mathrm{KL}}(p_{t}||p_{0})}{\partial p_{t}}\bigg] \\
&= -\nabla_{u}\cdot[p_{t}(\btheta,u)\hat{g_{1}}] - \nabla_{\btheta}\cdot[p_{t}(\btheta,u)\hat{g_{2}}] + \lambda \nabla_{u}\cdot[p_{t}(\btheta,u)u] \\
& \qquad + \lambda \nabla_{\btheta}\cdot[p_{t}(\btheta,u)\btheta] + \lambda \nabla \cdot[p_{t}\nabla\log(p_{t})]\\
&= - \nabla_{u}\cdot[p_{t}(\btheta, u)g_{1}(t, \btheta, u)] - \nabla_{\btheta}\cdot[p_{t}(\btheta, u)g_{2}(t, \btheta, u)]  + \lambda \Delta[p_{t}(\btheta, u)] \\
&= \frac{dp_{t}}{dt},
\end{align*}
where the third equation is by the definition $g_{1}(t, \btheta, u) =  \hat{g}_{1}(t, \btheta, u) - \lambda u$, $g_{2}(t, \btheta, u) =  \hat{g}_{2}(t, \btheta, u) - \lambda \btheta$ and $p_{t}\nabla\log(p_{t})=\nabla p_{t}$.
\end{proof}
\subsection{Proof of Lemma~\ref{lemma: I}}
Here we give the proof of Lemma~\ref{lemma: I}. 
\begin{proof}[Proof of Lemma~\ref{lemma: I}] The proof is based on the smoothness properties of $h(\btheta,\xb)$ given in Lemma \ref{lemma:new2oldassumption}. We have
\begin{align*}
&\big|\EE_{p}\big[\big(uh(\btheta,\xb) + u\nabla h(\btheta,\xb)\cdot \btheta - u\Delta h(\btheta, \xb) \big)\big]\big|\\
&\qquad \leq \EE_{p}\big[|u|G + G|u| \| \btheta\|_{2} + G|u|\big]\\
&\qquad =G\EE_{p}[|u|\|\btheta \|_{2}] + 2G\EE_{p}[|u|]\\
&\qquad \leq G\EE_{p}\bigg[\frac{u^2 + \|\btheta\|_{2}^2}{2}\bigg] + 2G\sqrt{\EE_{p}[u^2]},
\end{align*}
where the first inequality is by $|h(\btheta, \xb)| \leq G$,  $\|\nabla_{\btheta}h(\btheta, \xb)\|_{2} \leq G$ and $|\Delta h(\btheta, \xb)| \leq G$ in Lemma \ref{lemma:new2oldassumption}, the second inequality is by Young's inequality and  Cauchy-Schwartz inequality. Now by $\cW(p, p_{0})\leq \sqrt{d+1}$ and Lemma \ref{lemma: Variance}, we have 
\begin{align*}
&\Big|\EE_{p}\Big[\big(uh(\btheta,\xb) + u\nabla h(\btheta,\xb)\cdot \btheta - u\Delta h(\btheta, \xb) \big)\Big]\Big|\\
&\qquad\leq 2G(d +1) + 4G\sqrt{d+1}\\
&\qquad =A_{1}.
\end{align*}
This completes proof.
\end{proof}

\subsection{Proof of Lemma~\ref{lemma:Ddynamic}}

\begin{proof}[Proof of Lemma~\ref{lemma:Ddynamic}]
By Lemma \ref{lemma:chainrule}, we have
\begin{align}
\frac{\partial  D_{\mathrm{KL}}(p_{t}||p_{0})}{\partial t} &= \int_{\RR^{d+1}}\frac{\partial  D_{\mathrm{KL}}(p_{t}||p_{0})}{\partial p_{t}}\frac{dp_{t}}{dt}d\btheta du \notag\\
&=\int_{\RR^{d+1}}\frac{\partial  D_{\mathrm{KL}}(p_{t}||p_{0})}{\partial p_{t}}\nabla \cdot \bigg[p_{t}(\btheta, u)\nabla \frac{\partial Q(p_{t})}{\partial p_{t}}\bigg]d\btheta du \notag\\
&=-\int_{\RR^{d+1}}p_{t}(\btheta, u)\bigg[\nabla\frac{\partial  D_{\mathrm{KL}}(p_{t}||p_{0})}{\partial p_{t}}\bigg] \cdot \bigg[\nabla \frac{\partial Q(p_{t})}{\partial p_{t}}\bigg]d\btheta du \notag\\
&=-\lambda\int_{\RR^{d+1}}p_{t}(\btheta, u)\bigg[\nabla\frac{\partial  D_{\mathrm{KL}}(p_{t}||p_{0})}{\partial p_{t}}\bigg] \cdot \bigg[\nabla\frac{\partial  D_{\mathrm{KL}}(p_{t}||p_{0})}{\partial p_{t}}\bigg]d\btheta du \notag\\
&\qquad - \int_{\RR^{d+1}}p_{t}(\btheta, u)\bigg[\nabla \frac{\partial  D_{\mathrm{KL}}(p_{t}||p_{0})}{\partial p_{t}}\bigg]\cdot\bigg[\nabla\frac{\partial L(p_{t})}{\partial p_{t}}\bigg]d\btheta du \label{eq:1.4},
\end{align}
where the second and last equations are by Lemma \ref{lemma: gradient}, the third inequality is by applying integration by parts multiple times. We further calculate by Lemma \ref{lemma: gradient},
\begin{align}
& \int_{\RR^{d+1}}p_{t}(\btheta, u)\bigg[\nabla\frac{\partial  D_{\mathrm{KL}}(p_{t}||p_{0})}{\partial p_{t}}\bigg] \cdot \bigg[\nabla\frac{\partial  D_{\mathrm{KL}}(p_{t}||p_{0})}{\partial p_{t}}\bigg]d\btheta du \notag\\
&\qquad=\int_{\RR^{d+1}}p_{t}(\btheta,u)\|\btheta +  \nabla_{\btheta}\log(p_{t})\|^2 + \int_{\RR^{d+1}}p_{t}(\btheta,u)|u +  \nabla_{u}\log(p_{t})|_2^2\label{eq:1.5}.
\end{align}
Moreover, for the second term on the right-hand side of \eqref{eq:1.4} we have
\begin{align}
 &\int_{\RR^{d+1}}p_{t}(\btheta, u)\bigg[\nabla \frac{\partial  D_{\mathrm{KL}}(p_{t}||p_{0})}{\partial p_{t}}\bigg]\cdot\bigg[\nabla\frac{\partial L(p_{t})}{\partial p_{t}}\bigg]d\btheta du \notag\\
 &\qquad = \int_{\RR^{d+1}}p_{t}(\btheta, u) [-\hat{g}_{1}(t, \btheta, u)] \cdot [u + \nabla_{u}\log(p_{t})] d\btheta du \notag \\
&\qquad \qquad+ \int_{\RR^{d+1}}p_{t}(\btheta, u) [-\hat{g}_{2}(t, \btheta, u)] \cdot [ \btheta + \nabla_{\btheta}\log(p_{t})] d\btheta du \notag\\
&\qquad= -\int_{\RR^{d+1}}p_{t}(\btheta, u) [\hat{g}_{1}(t, \btheta, u) \cdot u + \hat{g}_{2}(t, \btheta, u) \btheta] d\btheta du \notag \\
&\qquad \qquad - \int_{\RR^{d+1}}[\hat{g}_{1}(t, \btheta, u)\cdot \nabla_{u}p_{t}(t, \btheta, u) + \hat{g}_{2}(t, \btheta, u)\cdot\nabla_{\btheta}p_{t}(t, \btheta, u)]d\btheta du \notag\\
&\qquad = -\int_{\RR^{d+1}}p_{t}(\btheta, u) [\hat{g}_{1}(t, \btheta, u) \cdot u + \hat{g}_{2}(t, \btheta, u) \btheta] d\btheta du \notag \\
&\qquad \qquad +\int_{\RR^{d+1}}p_{t}(t, \btheta, u)[\nabla_{u}\cdot\hat{g}_{1}(t, \btheta, u) + \nabla_{\btheta}\cdot\hat{g}_{2}(t, \btheta, u)]d\btheta du,\label{eq:1.6}
\end{align}
where the second equation is by $p_{t}\nabla\log(p_{t}) = \nabla p_{t}$ and the third equation is by applying integration by parts. 
Then plugging \eqref{eq:1.5} and \eqref{eq:1.6} into \eqref{eq:1.4}, we get
\begin{align*}
\frac{\partial  D_{\mathrm{KL}}(p_{t}||p_{0})}{\partial t} &= -\lambda\int_{\RR^{d+1}}p_{t}(\btheta,u)\|\btheta +  \nabla_{\btheta}\log(p_{t})\|_2^2 - \lambda\int_{\RR^{d+1}}p_{t}(\btheta,u)|u +  \nabla_{u}\log(p_{t})|^2 \notag\\
&\qquad + \int_{\RR^{d+1}}p_{t}(\btheta, u)[\hat{g}_{1}\cdot u + \hat{g}_{2}\cdot \btheta - \nabla_{u} \cdot \hat{g}_{1} - \nabla_{\btheta} \cdot \hat{g}_{2} ] d\btheta du.
\end{align*}
This completes the proof.
\end{proof}

\subsection{Proof of Lemma~\ref{lemma:diffusion term2}}

\begin{proof}[Proof of Lemma~\ref{lemma:diffusion term2}]
We remind the readers the definitions of $\hat g_1$ and $\hat g_2$ in \eqref{eq:def_ghat1} and \eqref{eq:def_ghat1}. We have
\begin{align*}
&\int_{\RR^{d+1}}p_{t}(\btheta, u)[\hat{g}_{1}\cdot u + \hat{g}_{2}\cdot \btheta - \nabla_{u} \cdot \hat{g}_{1} - \nabla_{\btheta} \cdot \hat{g}_{2}] d\btheta du  \\
&\qquad= 2\alpha\EE_{S}\bigg[(f(p_{t}, \xb)-y)\int_{\RR^{d+1}}\big(uh(\btheta,\xb) + u\nabla_{\btheta} h(\btheta,\xb)\cdot \btheta - u\Delta h(\btheta, \xb) \big)p_{t}(\btheta, u)d \btheta du\bigg].
\end{align*}
Denote $I(\btheta, u, \xb) = uh(\btheta,\xb) + u\nabla_{\btheta} h(\btheta,\xb)\cdot \btheta - u\Delta h(\btheta, \xb)$,
then we have 
\begin{align}\label{eq:diffusion term2_proof_bound1}
|\nabla_{u}I(\btheta, u, \xb)| &= |h(\btheta,\xb) + \nabla_{\btheta} h(\btheta,\xb)\cdot \btheta - \Delta h(\btheta, \xb)| \leq G\|\btheta\|_{2} + 2G,
\end{align}
where the inequality holds by Lemma~\ref{lemma:new2oldassumption}. Similarly, we have
\begin{align}\label{eq:diffusion term2_proof_bound2}
\|\nabla_{\btheta}I(\btheta, u, \xb)\|_{2} &= \|u\nabla_{\btheta} h(\btheta, \xb) + u \nabla_{\btheta}\big(\nabla_{\btheta}h(\btheta, \xb)\cdot \btheta\big) - u\nabla_{\btheta} \Delta_{\btheta}h(\btheta, \xb)\big)\|_{2} \nonumber \\
&\leq 3G|u|.
\end{align}
Therefore, combining the bounds in \eqref{eq:diffusion term2_proof_bound1} and \eqref{eq:diffusion term2_proof_bound2} yields
\begin{align*}
\sqrt{\nabla_{u}I(\btheta, u, \xb)^2 + \| \nabla_{\btheta}I(\btheta, u, \xb) \|_2^2} \leq 4G\sqrt{u^2 + \|\btheta\|_{2}^2} + 2G.
\end{align*}
By Lemma \ref{lemma: second-order},
we have that
\begin{align*}
\EE_{p_{t}}[I(\btheta_{t}, u_{t}, \xb)] - \EE_{p_{0}}[I(\btheta_{0}, u_{0}, \xb)] &\leq \Big[8G\sqrt{d+1} + 2G\Big]\mathcal{W}(p_{0}, p_{t})\\
&\leq A_{2}\sqrt{ D_{\mathrm{KL}}(p_t||p_{0})},
\end{align*}
where the last inequality is by Lemma \ref{lemma:W2toKL} and $A_{2} = 16G\sqrt{d+1} + 4G$.
By $\EE_{p_{0}}[I(\btheta_{0}, u_{0}, \xb)] = \EE_{p_{0}}[u_{0}]\EE_{p_{0}}[h(\btheta_{0},\xb) + \nabla_{\btheta} h(\btheta_{0},\xb)\cdot \btheta_{0} - \Delta_{\btheta} h(\btheta_{0}, \xb)] = 0$, we further have 
\begin{align}
\EE_{p_{t}}[I(\btheta_{t}, u_{t}, \xb)] \leq A_{2}\sqrt{ D_{\mathrm{KL}}(p_t||p_{0})}. \label{eq: I}  
\end{align}
Then we have
\begin{align*}
&\int_{\RR^{d+1}}p_{t}(\btheta, u)[\hat{g}_{1}\cdot u + \hat{g}_{2}\cdot \btheta - \nabla \cdot \hat{g}_{1} - \nabla \cdot \hat{g}_{2}] d\btheta du  \\
&\qquad= 2\alpha\EE_{S}\big[(f(p_{t},\xb)-y)\EE_{p_{t}}[I(\btheta_{t}, u_{t}, \xb)]\big]\\
&\qquad\leq 2\alpha A_{2} \sqrt{ D_{\mathrm{KL}}(p_{t}||p_{0})}\sqrt{L(p_{t})},
\end{align*}
where the last inequality is by \eqref{eq: I} and Cauchy-Schwarz inequality. This completes the proof.
\end{proof}